\documentclass[twoside,11pt]{article}
\usepackage{jmlr2e}

\usepackage{algorithm}  
\usepackage{algpseudocode}  
\usepackage{amsmath}
\usepackage{bbm}
\usepackage{multirow}
\usepackage{chngpage}
\usepackage{verbatim}
\usepackage{mathtools}
\usepackage{booktabs}
\usepackage{float}
\usepackage{array}
\newcolumntype{P}[1]{>{\centering\arraybackslash}p{#1}}
\usepackage{tabularx}
\usepackage{enumitem}
\usepackage{caption}
\usepackage{multicol}
\newtheorem{assumption}{Assumption}

\usepackage{xcolor}
\usepackage{hyperref}
\hypersetup{
    colorlinks=true,
    linkcolor=blue,
    filecolor=magenta,      
    urlcolor=cyan,
    citecolor={blue}
}

\firstpageno{1}

\title{Effective Federated Adaptive Gradient Methods with Non-IID Decentralized Data}
\author{\name Qianqian Tong \email qianqian.tong@uconn.edu \\
       \name Guannan Liang \email guannan.liang@uconn.edu \\
       \name Jinbo Bi \email jinbo.bi@uconn.edu\\
       \addr Computer Science and Engineering\\University of Connecticut,
       Storrs, CT 06269}

\begin{document}
\editor{ }
\maketitle

\begin{abstract}
Federated learning allows loads of edge computing devices to collaboratively learn a global model without data sharing. The analysis with partial device participation under non-IID and unbalanced data reflects more reality. In this work, we propose a family of effective federated adaptive gradient methods (AGMs), 
to alleviate generalization performance deterioration caused by dissimilarity of data population among devices. Under our proposed framework, we compare several schemes of calibration for the adaptive learning rate, including the standard Adam and AMSGrad calibrated by $\epsilon$, $p$-Adam, and ones calibrated by an activation function. Our analysis provides the first set of theoretical results that the (calibrated) Federated AGMs , which employ both the first-order and second-order momenta,  converge to a first-order stationary point under non-IID and unbalanced data settings for nonconvex optimization.
We perform extensive experiments to compare these federated learning methods with the state-of-the-art FedAvg, FedMomentum and SCAFFOLD and to assess the different calibration schemes and the advantages of AGMs over the current federated learning methods. 
\end{abstract}

\section{Introduction}
Federated learning (FL) is a privacy-preserving learning framework for large scale machine learning on edge computing devices, and solves the data-decentralized distributed optimization problem:
\begin{align}\label{Problem1}
\min_{x\in \mathbb{R}^d}f(x) =  \sum_{i=1}^N p_i f_{i}(x), 
\end{align}
where $f_i(x) = E_{z \sim \mathcal{D}_i}[f_i(x, z)]$ is the loss function of the $i^{th}$ client (or device) with weight $p_i \in [0,1)$, $\sum_{i=1}^N p_i =1$, $\mathcal{D}_i$ is the distribution of data located locally on the $i^{th}$ client, and $N$ is the total number of clients. 
FL enables numerous clients to coordinately train a model parameterized by $x$, while keeping their own data locally, rather than sharing them to the central server (\cite{konevcny2016federated,konevcny2016federated1}).

Compared with existing well-studied distributed computing, there are four main key differences (\cite{konevcny2016federated,kairouz2019advances,sattler2019robust,li2019federated,li2019convergence}): 1) the number of clients $N$ is large and the communication between clients and central server can be slow;
2) partial device participation is allowed during 
model training, e.g., some devices may randomly drop out or come back during the training phase; 
3) the distributions of  training data over clients are non-independent and non-identically distributed (non-IID), i.e., local data on each client cannot be regarded as samples IID drawn from an overall distribution; 4) the data quantities and weights can be unbalanced across devices. Some studies assume  $p_i= \frac{1}{N}$, but in general $p_i$ can differ on different clients.
The clients/devices can be smartphones, personal computers, network sensors, or other accessible information resources. The data decentralization is important during model training because the data on each edge computing device can be private and sensitive, such as photos or personal conversations. FL has attracted widespread attentions, and how to effectively solve Problem (\ref{Problem1}) is essential.

The FedAvg algorithm is proposed in  (\cite{mcmahan2016communication}) and has become the de facto FL algorithm where clients do not communicate with the central server at each iteration but after $K$ inner iterations. In the $t$-th round, a client obtains $ x^{(i)}_{t , k+1} = \underset{x}{ \text{argmin}} \: f_i(x^{(i)}_{t , k}) +\langle \nabla f_i(x^{(i)}_{t, k}), x - x^{(i)}_{t, k} \rangle + \frac{1}{2\gamma} \|x - x^{(i)}_{t, k}\|^2 $ for $k \in  \{0, ..., K-1\}$ and sends $x^{(i)}_{t , K}$ to the central server, which then averages all the updates from the clients to obtain the global update $x_{ t+1} = \sum_{i =1 }^{N} p_i x^{(i)}_{t, K}$ to broadcast to all clients in the $(t+1)$-th round. FedAvg can significantly reduce the communication cost. 
It however has been identified to suffer from the client drift issue under non-IID data situation (\cite{hsu2019measuring,karimireddy2019scaffold,reddi2020adaptive}). The average of  client updates could drift to $\frac{1}{N} \sum_{i = 1}^{N} x^*_i$, rather than $x^*$, where $x^*$ and $x^*_i$ are the optimal solutions, respectively, to Problem (\ref{Problem1}) and to the problem of $ \min_x f_i(x) $ for the $i^{th}$ client. This issue becomes exacerbated when partial device participation is present. 
Hsu et al. experimentally explore the performance of FL algorithms on visual classification tasks and show that FedAvg performs worse with increasing non-IIDness than FedMomentum, which can consistently improve the test accuracy (\cite{hsu2019measuring,hsu2020federated}) and theoretically converge to a first-order stationary point (\cite{huo2020faster}). 
Karimireddy et al. propose SCAFFOLD (\cite{karimireddy2019scaffold}) to use variance reduction technique in the local $K$ inner iterations to alleviate the effect of client drift.
The FedProx algorithm is proposed in (\cite{li2018federated}) where the $i^{th}$ client tries to add a proximal term $\frac{\mu}{2} \|x - x_{t}\|^2$ to the local subproblem  to effectively limit the impact of variable local updates, thus keeping local updates close to the global iterate. 
Recently, an adaptive framework that uses an adaptive learning rate on the server side of FL is developed in (\cite{reddi2020adaptive}). They empirically use momentum for FedAdam and FedYogi, and highlight the benefit of momentum, however, the theoretical analysis only explores the second-order momentum without the accumulated gradients (the first-order momentum).
 

\begin{figure}[t!]
\centering
  \includegraphics[trim=20 0 10 0,clip,width=\textwidth]{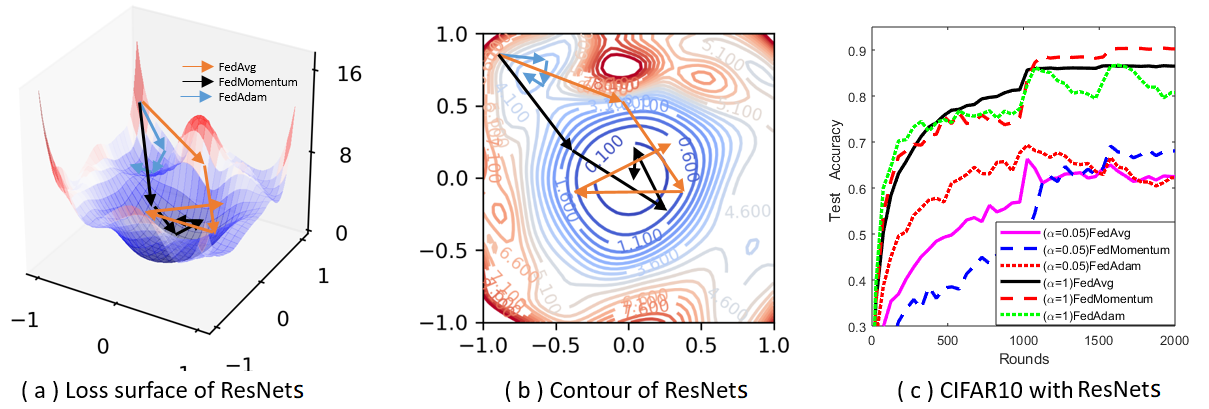}  
  \caption{Behaviors of FedAvg, FedMomentum and FedAdam. (a) and (b) are 3D loss surface and contour of ResNets with CIFAR10 data around one flat local minimal (\cite{li2018visualizing}). 
  (c) shows the learning curves of the different algorithms with a multi-stage learning rate decay at different non-IID levels across 100 clients as specified by $\alpha$. The concentration parameter $\alpha$ in the Dirichlet distribution controls the degree of data dissimilarity across devices. Smaller $\alpha$ values correspond to higher levels of non-IID data distributions among clients. 
  }
  \label{fig:sub1}
\end{figure}


We examine theoretically and experimentally here the federated adaptive gradient methods (AGMs), including federated Adam and federated AMSGrad methods, with both the first-order momentum as search direction and the second-order momentum as adaptive learning rate (or stepsize), to solve Problem (\ref{Problem1}).
Figure \ref{fig:sub1} is a summary illustration which shows Federated AGMs are usually more efficient in early training stages (e.g., 1-1000 rounds for ResNets with CIFAR10 data), especially for non-IID data.  However, even equipped with a multi-stage learning rate scheme (\cite{chen2018universal,yang2018does}), a direct merge of Adam and FL  tends to trap in a narrow (bad) local minimal, leading to lower training accuracy.

Under FL, the stochastic  direction on the central server (virtual direction $\Delta_t$ in Algorithm \ref{alg:Fedadam}) aggregated by clients tends to be very small when approaching local minimal; thus, momenta will be small and the element-wise division of momenta could be inconsistent, due to the different speed of moving averages of momenta. What's more, the adaptive stepsize in FL may have a large span across coordinates.
To alleviate these issues, the base learning rate for FedAdam is often set smaller than it should be. However, similar to (\cite{li2018visualizing}), with a very small learning rate, Federated AGMs may get stuck at narrow and sharper local minimals, which exhibit bad generalization performance.
Hence, we propose to use several calibration schemes (\cite{zaheer2018adaptive,chen2018closing,tong2019calibrating}) to calibrate the adaptive learning rate in Federated AGMs.

Our main results are summarized as follows: 
\begin{enumerate}
    \item  We characterize the performance of Federated AGMs (FedAdam and FedAMSGrad): rapid progress in early stages but prone to be stuck at narrow local minimal with non-IID dataset. 
    \item We propose a calibration framework for Federated AGMs, which includes  $\epsilon$-calibrated, $p$-calibrated, and a calibration method using the activation function {\em softplus}, denoted by $s$-calibrated. This effort also unifies FedMomentum and Federated AGMs so to achieve the best of both approaches for fast convergence as well as good generalization performance.
    \item Theoretical analysis based on more practical reality - partial device participation over unbalanced and non-IID datasets for nonconvex objective functions - shows that convergence rate is highly related to calibration parameters,  gradient dissimilarity, and the interplay between the learning rate, the number of local iterations, and the quantity of participated clients.
    \item Experimental results show that the calibrated Federated AGMs equipped with stagewise local learning rate decay can achieve the best performance in both training and test accuracy on multiple FL tasks over the state of the art.
\end{enumerate}

\noindent\textbf{Notations.} For any vectors $a,b \in \mathbb{R}^d$, we use $a\odot b$ for element-wise product, $a^m$ for element-wise power of $m$, $\sqrt{a}$ for element-wise square root, $a/b$ for element-wise division, and $\langle a,b \rangle$ to denote the inner product of $a$ and $b$. We use $x^{(i)}$ to denote the parameter update on the $i$-th device, and $\|x\|$ to denote the $l_2$-norm of $x$.  Let $N$ denote the total number of clients, $[N]$ denote the integer set $\{1,2,...,N\}$, and $S (\leq N)$ be the maximum number of participated clients, $T$ be the total rounds that the sever updates its global model, $K$ be the number of local updates on each client, $O(\cdot)$ hide constants which do not rely on the problem parameters, $\Theta(\cdot)$ denote the same order of computation.
 
\section{Methods}

\subsection{AGMs Revisited.}

Distributed learning with data $z$ collected to the central server can be formulated as the minimization problem of: $ \min_{x\in \mathbb{R}^d}f(x) =  E_{z \sim \mathcal{D}}[f(x, z)]$, where both $f(x)$ and $f(x, z)$ are usually nonconvex. The SGD, its momentum variants  (\cite{ghadimi2013stochastic,wright1999numerical,wilson2016lyapunov,yang2016unified}), and its adaptive versions (AGMs) (\cite{duchi2011adaptive,kingma2014adam,zeiler2012adadelta}) can readily distribute their computation to multiple processors due to their stochasticity and simplicity. The updating rule of these methods can be generally written as: $x_{t+1}=x_{t}- \frac{\eta_t}{\sqrt{v_t}}\odot m_t$, where $\odot$ calculates element-wise product of the first-order momentum $m_t$ and the learning rate  $\frac{\eta_t}{\sqrt{v_t}}$. Here we call $\eta$ the base learning rate, $\frac{1}{\sqrt{v_t}}$ the adaptive learning rate. 
Researchers generally have an agreement on how to compute $m_t$ to accelerate the convergence, i.e., $m_t=\beta_1 m_{t-1}+(1-\beta_1)g_t$, $\beta_1\in [0,1)$, but there are various formulas for the second-order momentum $v_t$ and the related adaptive learning rate in \textsc{Adagrad} (\cite{duchi2011adaptive}), \textsc{Adadelta} (\cite{zeiler2012adadelta}), \textsc{Rmsprop} (\cite{tieleman2012lecture}), Adam (\cite{kingma2014adam}), AMSGrad (\cite{reddi2018convergence}), \textsc{Yogi} (\cite{zaheer2018adaptive}) and \textsc{AdaBound} (\cite{luo2019adaptive}).
Among these methods, Adam uses exponential moving averages of past squared gradients, i.e., $v_t = \beta_2 v_{t-1} +(1-\beta_2) g_t^2, \beta_2 \in [0,1).$ AMSGrad takes the larger second-order momentum estimated in the past iterations by $v_t = \max\{v_{t-1}, \hat{v}_{t}\}$, where $\hat{v}_{t} = \beta_2 v_{t-1} +(1-\beta_2) g_t^2$ to theoretically ensure convergence. In this paper, we mainly focus on Adam and AMSGrad, but our analysis is readily applicable to other AGMs.

\subsection{Federated AGMs. } 

For Federated AGMs (in Algorithm \ref{alg:Fedadam}) based on  the classical SGD in inner loops and the AGMs in outer loops,
we provide theoretical analysis and empirical studies in more practical reality: (1) partial device participation, which means that in each communication round, only a portion of the $N$ clients is active. Because clients are assumed to randomly leave or join the FL, we can randomly sample a set of clients ($S_t \subset [N]$). Different clients may have different weights $p_i$ in the FL, so active clients are sampled according to $p_i$'s. When the cardinality $S=|S_t|$ is $N$, all clients participate. (2) Local devices are presumably unbalanced in the capability of curating data, i.e., that clients may exhibit different amounts of local data as specified by $p_i$'s and become balanced if $p_i= \frac{1}{N}$. (3) There can be different levels of non-IID data across clients, as discussed in (\cite{hsu2019measuring}), which provides a simulation process. We draw a distribution $q$ 
with Dirichlet distribution, i.e., $q\sim Dir(\alpha \mathbb{I})$, where $\mathbb{I}=(1, .., 1)$ is a prior class distribution over a pre-specified number of classes and $\alpha$ is the so-called concentration parameter. When $\alpha \rightarrow 0$, it generates strongly different distributions among clients. When $\alpha \rightarrow \infty$, all clients have identical distributions to the prior. After $q$ is drawn, data for each class will be drawn according to $q$.

\begin{algorithm}[t]
\captionof{algorithm}{Federated AGMs:}
\label{alg:Fedadam}
\begin{algorithmic}
         \State {\bfseries Input:} The SGD learning rate $\gamma_t$, the AGM base learning rate $\eta$, momentum parameters $0\leq \beta_1, \beta_2 < 1$, the number of clients $N$, the number of inner iterations $K$. 
        \State {\bfseries Initialize} $x_{0}$ randomly, and $m_0 = 0$, $v_0 = 0$
        \For {$t = 0$ to $T-1$}
        \State Sample $S_t \subset [N]$ where $|S_t|=S$ based on probability \\$\;\;\;\;$ $\{p_i\}$ with replacement
        \For {client $i=1 \in S_t$ parallel}
        \State Receive $x_t$ from the central server, set $x_{t,0}^{(i)} = x_t$
        \For {$k = 0$ to $K-1$}
        \State Base optimizer step: $x_{t,k+1}^{(i)} = x_{t,k}^{(i)} - \gamma_t g_{t,k}^{(i)} $
        \EndFor
        \State Send local $x_{t,K}^{(i)}$ to the central server.
        \EndFor
        \State Exact-Average: $\tilde{x}_{t+1} = \frac{1}{S}\sum_{i\in S_{t}} x_{t,K}^{(i)}$  
        \State Virtual direction: $\Delta_t =  x_{t} - \tilde{x}_{t+1}$
        \State Update the first-order momentum:   $m_t = \beta_1 m_{t-1}+(1-\beta_1)\Delta_t$
        \State Update the second-order momentum:   
        \State(FedAdam): $v_t = \beta_2 v_{t-1}+ (1-\beta_2)\Delta_t^2$  
        \State (FedAMSGrad): $\hat{v}_t = \beta_2 v_{t-1}+ (1-\beta_2)\Delta_t^2$,  
        $v_t = \max\{\hat{v}_t, v_{t-1}\}$
        \State AGM update: $x_{t+1}=x_{t}-\frac{\eta}{\mbox{calibrate}({v_t})}\odot m_t$
        \EndFor
\end{algorithmic} 
\end{algorithm}

Algorithm \ref{alg:Fedadam} depicts the steps of Federated AGMs in a nested loop structure, the outer loops/communication rounds require message passing between the central server and active devices.
Each round corresponds to parallel inner loops that run within individual clients. 
The client $i$ updates $x_{t,k}^{(i)}$ at the $k$-th inner iteration of the $t$-th outer round, and performs $K$ steps of SGD without communication to other clients. The final local update $ x_{t,K}^{(i)} = x_{t,0}^{(i)} - \gamma_t g_t^{(i)}$ where $ g_t^{(i)} = \sum_{k=0}^{K-1} g_{t,k}^{(i)}$. The number of inner iterations $K$ should be a small number so that each client does not move too far away independently. In experiments, we set it to a number that makes sure a full epoch is finished (i.e., the amount of available data on client $i$ divided by the mini-batch size).
The central server then aggregates the local updates by averaging $\tilde{x}_{t+1} =  \frac{1}{S}\sum_{i\in S_{t}} x_{t,K}^{(i)}$. Then the virtual direction amounts to computing the sum of all the gradients obtained in the inner loops $\Delta_t = x_{t} - \tilde{x}_{t+1} =   \frac{1}{S}\sum_{i\in S_{t}} (x_{t} - x_{t,K}^{(i)})=   \frac{1}{S}\sum_{i\in S_{t}} (x_{t,0}^{(i)} - x_{t,K}^{(i)})= \frac{1}{S}\sum_{i\in S_{t}}\sum_{k=0}^{K-1} \gamma_t g_{t,k}^{(i)}$.

Algorithm \ref{alg:Fedadam} specifies the AGM steps in outer loops. For inner loops, although Algorithm \ref{alg:Fedadam} instantiates with SGD, other optimizers can be applicable, e.g., the gradient descent (GD) using full local on-device data, or stochastic variance reduced gradient (SVRG) method. The recent SCAFFOLD (\cite{karimireddy2019scaffold}) uses variance reduction technique in a FedAvg framework. We can also use SCAFFOLD in inner loops, which we have compared in our experiments.
The inner stepsize $\gamma_t$ is predefined and can differ from round to round. Our theoretical analysis shows that, to guarantee convergence, the total local aggregate $K\gamma_t$ in each client's update should be restricted to be constant, to warrant that the local alteration is not too far off across rounds. Thus the inner stepsize should be small if $K$ is large or vice versa.

In an outer step of AGM update, both momenta $m_t$ and $v_t$ are used to track more past information. The base learning rate $\eta$ is predefined.
The adaptive step size is determined by $v_t$ via a specific calibration, which we will discuss in the next subsection. 
When the $\text{calibrate}(v_t)$ function takes the form of $\sqrt{v_t} + \epsilon $ where $\epsilon>0$ is small, Algorithm 1 gives us the Federated Adam (FedAdam). In Figure \ref{fig:sub1} (c), FedAdam ($\epsilon = 10^{-8}$ as commonly suggested) shows rapid initial progress but in later stages, has worse accuracy than other methods. 
With the increasing level of non-IID, FedAdam remains its fast convergence rate, but it tends to be more likely trapped in narrow local minimal, and the final accuracy is often surpassed by FedMomentum. Our study shows that careful calibration of the adaptive stepsize can be important for Federated AGMs to improve performance in later stages.

\subsection{Different Calibration Schemes.}

\begin{figure*}[t!]
\centering
  \includegraphics[width=\textwidth]{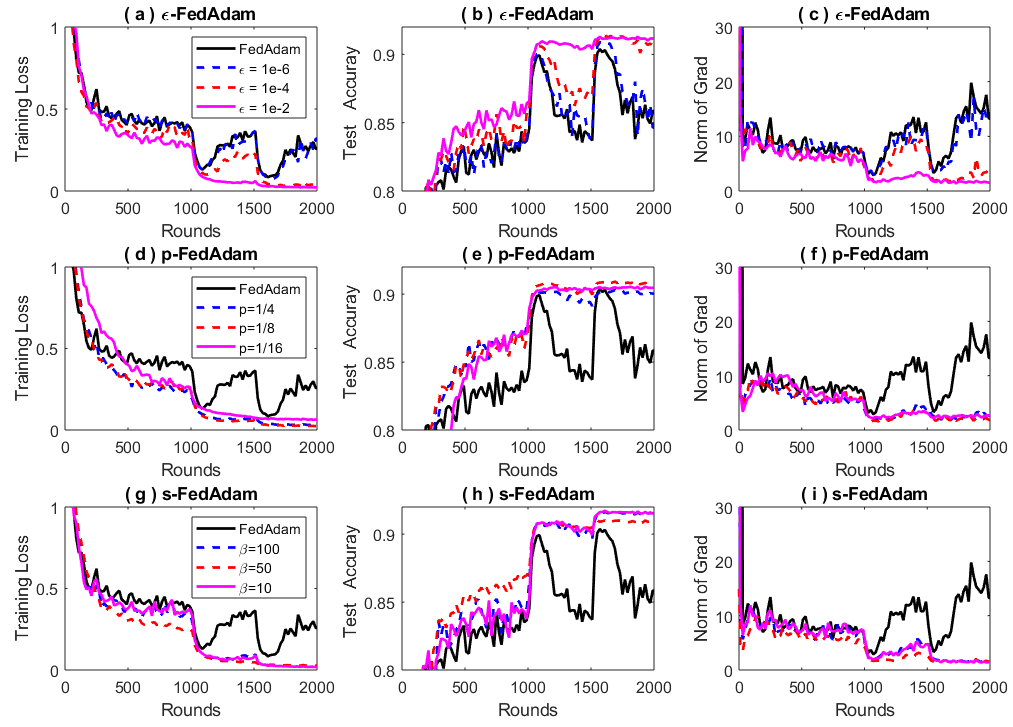}
  \caption{The training loss (a, d, g), test accuracy (b, e, h) and the norm of gradients (c, f, i) for $\epsilon$-, $p$-, $s$-FedAdam with a multi-stage decay of $\gamma$ for ResNets on the CIFAR10 dataset. After $\gamma$ is decayed to $0.1*\gamma$ at round 1000, the test accuracy of the FedAdam ($\epsilon = 10^{-8}$) first increases quickly and then gets stuck at a sharp local minimal in the few steps before round 1500.  At rounds 1500 and 2000, $\gamma$ is further decayed. The test accuracy degenerates again. Monitoring the norm of the gradients in the plots on the last column helps us examine the local areas of the stationary points. Without careful calibration, FedAdam is stuck at a sharp local minimal after round 1000 (the gradient norms vary a lot in the proximity of a sharp local minimal). However, the $p$- and $s$-FedAdam converge to a flat local minimal. }
  \label{fig:sub2}
\end{figure*}

 As a special intermediate product in FL, the virtual direction is decided by clients number, inner stepsize $\gamma_t$, and stochastic gradient of each client. We theoretically prove that $\gamma_t$ is a small amount, and stochastic gradient will close to zero when the federated algorithm converges; thus, the aggregated virtual direction ( $\Delta_t$ in Algorithm \ref{alg:Fedadam}) will be approximating to zero. Hence the base learning rate $\eta$ needs to be set small to avoid violent fluctuation caused by adaptive stepsize. On the other hand, a small $\eta$ is hard to escape the narrow local minimum.
This becomes worse with the increasing dissimilarity of data distributions across clients.
When the training/test accuracy curve reaches a plateau in the later stages, Federated AGMs may reach a too small $\eta$ to jump out of a sharp local minimum, and sharp local minimizers (e.g., Fig. 1 (b) top-left corner) often have worse generalization performance than the optimizer whose neighborhood is flat and convex. 

In this work, we further provide a careful examination of three calibration techniques that all regulate the span of the adaptive stepsizes under FL.
We take FedAdam as an example to discuss these calibration formulae and present its corresponding performance, but the FedAMSGrad can certainly benefit from them as well (as shown in Appendix).

\subsubsection{$\epsilon$-FedAdam.} 
Like the Adam method, the FedAdam can also use the hyper-parameter $\epsilon$ to avoid the denominator of adaptive learning rate, $\frac{1}{\sqrt{v_t} + \epsilon}$ from vanishing to zero. 
The value of $\epsilon$ decides the largest span or dissimilarity of the adaptive stepsizes (learning rates) over the coordinates. For instance, if $\epsilon = 10^{-8}$, the smallest and largest adaptive stepsizes in a single iteration could range in $(0, 10^8)$. If we calibrate the adaptive stepsizes directly by controlling $\epsilon$, we call the method the “$\epsilon$-FedAdam”. When we choose $\epsilon \in \{10^{-2}, 10^{-4}, 10^{-6}\}$, we clearly observe different performance in Figure \ref{fig:sub2} (a,b,c). The similar empirical results for “$\epsilon$-FedAMSGrad" are included in Appendix.

\subsubsection{$p$-FedAdam.}
Another choice of calibration replaces the square root of $v_t$ in the calibration$(v_t)$ function by the $p$-th root of $v_t$. In other words, calibrate$(v_t) = (v_t +\epsilon)^{p}$ where $p \in \{\frac{1}{2},\frac{1}{4},\frac{1}{8},\frac{1}{16} \}$ and $\epsilon = 10^{-8}$. It was shown in (\cite{chen2018closing}) that this calibration improves the test accuracy over the Adam that uses $\epsilon$. We can similarly calibrate the adaptive stepsize of the FedAdam by controlling $p$, which we call “$p$-FedAdam".  When $p$ is smaller, more compression is placed on the second-order momentum $v_t$, so it allows the base learning rate of the FedAdam to choose a larger value.

The above two kinds of calibration are  straightforward and somehow efficient in constraining the adaptive stepsize to a moderate range. However, they lose “adaptivity" to some degree because increasing $\epsilon$ or decreasing $p$ both work on all coordinates without distinction.
For instance, if $v_t$ ranges from $10^{-16}$ to $10^8$ over the coordinates, setting $\epsilon = 10^{-2}$ will constrain adaptive stepsizes of all coordinates into $(10^{-4}, 10^2)$ in $\epsilon$-FedAdam, and changing $p=\frac{1}{2}$ to $p=\frac{1}{4}$ will change the adaptive stepsizes from the range $(10^{-4}, 10^{8})$ to $(10^{-2},10^4)$ in $p$-FedAdam. 

\subsubsection{$s$-FedAdam.}
A method explored in (\cite{tong2019calibrating}) uses the property of the $softplus$ function (actually any suitable activation function) to construct calibration function, which can be more effective to solve the issue of Federated AGMs. So the calibrate$(v_t)$ becomes $softplus(\sqrt{v_t}) = \frac{1}{\beta} \log(1+ e^{\beta \sqrt{v_t}})$. This calibration brings some benefits: smoothing out extremely small $v_t$ rather than hard thresholding while keeping moderate stepsizes untouched with an appropriate $\beta$; removing the parameter $\epsilon$ because the {\em softplus} function can be lower-bounded by a nonzero number, $softplus(\cdot)\geq\frac{1}{\beta} \log2$. When $\beta$ is small, FedAdam behaves similar to FedMomentum; if $\beta$ is chosen to be very large, it becomes similar to the standard Adam.
Figure \ref{fig:sub2} shows the cases for $\beta \in \{10,50,100\}$.

In summary, the above three calibration methods are all easy to implement in the Federated AGMs, and tuning the calibration parameters does improve performance. Based on our empirical observations, the $s$-FedAdam/AMSGrad with $\beta=50$ almost always gave the best performance among all calibration techniques.

\section{Theoretical Analysis}

In previous FL studies, theoretical analysis is performed under assumptions that the data are i.i.d. drawn, and/or all devices are active and have equal computing and storage ability. 
However, in reality, devices like smartphones have limited memory and battery, and they may run out of battery, so drop out of the computation.
It is thus more practical to consider the case that only a partial collection of devices  participate in the FL at a time.  Hence, our analysis copes with these realistic situations under the following fairly standard assumptions used in the analysis of nonconvex  optimization.

\begin{assumption}\label{assumption}
The loss $f_{i}$ and the objective $f$ satisfy:
 \renewcommand{\labelenumii}{\Roman{enumii}} 
 \begin{enumerate}
   \item \textbf{L-smoothness.} $\forall x, y \in \mathbb{R}^d$,  $\forall i \in \{1,...,n\}$, $\|\nabla f_{i}(x)-\nabla f_{i}(y)\|\leq L \|x-y\|$.
   \item \textbf{Gradient bounded.} $\forall x \in \mathbb{R}^d$,  $\forall i \in \{1,...,n\}$, $\|\nabla f_i(x)\|^2 \leq G_i^2$, $G_i\geq 0$.
   \item \textbf{Variance bounded.} $\forall x_{t,k}^{(i)}\in \mathbb{R}^d$, $t\geq 1$, $E[g_{t,k}^{(i)}]=\nabla f_i(x_{t,k}^{(i)})$, $E[\| g_{t,k}^{(i)}-\nabla f_i(x_{t,k}^{(i)})\|^2]\leq \sigma_i^2$.
 \end{enumerate}
\end{assumption}

 Following the same convention in FL (\cite{li2018federated,karimireddy2019scaffold}), to measure the level of non-iid across clients, we assume the dissimilarity between the gradients of the local functions $f_i$ and the global function $f$ is bounded as follows. 
 
\begin{assumption}\label{dissimilarity}
 Bounded gradient dissimilarity (\textbf{$\sigma_g^2$-BGD}): $\forall x\in \mathbb{R}^d$,~ $\forall i \in \{1,...,n\}$, ~$E[\| \nabla f_i(x)-\nabla f(x)\|^2]\leq \sigma_g^2$.
\end{assumption}

\begin{assumption} \label{scheme2}
(\cite{sahu2018federated})
Let $S_t$ be the set of the active devices in the $t$-th round, obviously $S_t \subset [N]$. Assume $S_t$ contains a subset of $|S_t| = S$ nodes randomly selected with replacement according to the sampling probability $p_i=\frac{n_i}{\sum_{i=1}^N n_i}$, where $n_i$ is the number of samples located on client $i$. Assume that the devices' capabilities are unbalances, i.e., $p_i$'s can be distinct for different $i$. 
\end{assumption}

Under Assumption \ref{scheme2}, the exact-average step can be computed as $\tilde{x}_{t+1}= \frac{1}{S}\sum_{i\in S_{t}} x_{t,K}^{(i)}$. 
In our analysis, we use $g_{t,k}$ as accumulated gradient direction that from all participating devices at the $k$-th iteration of $t$-round, 
$
    g_{t,k}=\frac{1}{S}\sum_{i\in S_t}  g_{t,k}^{(i)};
$
and the virtual direction computed on  the server in Algorithm 1 can also be written as 
$
    \Delta_t 
      = \sum_{k=0}^{K-1}\gamma_t g_{t,k}.
$
Before giving the main theorem of the proposed algorithms, we first establish the following lemma for the variance of $g_{t, k}$, which plays an important role in our theoretical analysis.

\begin{lemma}\label{g_t}
Let Assumptions 1 and 2 hold. We have the following properties:
\begin{align*}
    E[g_{t,k}]=\sum_{i=1}^{N} p_i \nabla f_i (x_{t,k}^{(i)});&\\
E[\|g_{t,k}\|^2]\leq \underbrace{   \frac{1}{S} (12 \sum_{i=1}^N p_i \sigma_i^2 + 24  \sum_{i=1}^{N} p_i G_i^2)}_{partial~ participation}&
    + \underbrace{4\sum_{i=1}^N p_i  (\sigma_i^2 + G_i^2).}_{local\;\; updates}
\end{align*}
\end{lemma}

This lemma shows that the expected update direction  $E[g_{t,k}]$ is the sample mean of each local update direction with the probabilities $p_i$, and the second moment of the direction $g_{t,k}$ can be bounded by the sum of two terms: 1) the variance due to partial client participation, and 2) the variance caused by the stochasticity in the local SGD updates. The more clients participate, the smaller the first term will be. 
Moreover, Lemma \ref{g_t} also serves as a transmission for the influence of dissimilarity across clients to the convergence performance.

\begin{theorem} \label{th:p2} 
Let all assumptions hold, and $L, \sigma_i, G_i, \beta_1 , \beta_2$ be defined therein. 
Let $\mu_{lower}$ and $\mu_{upper}$ be the lower bound and upper bound, respectively, for the adaptive stepsizes $1/calibrate(v_t)$, $t\in[0, T-1]$ be the index of communication rounds, the  total number of iterations be $TK$.  
Then with an appropriately chosen inner stepsize $ \gamma_t < \min\{ \frac{1}{8LK},\frac{1}{K}\sqrt{\frac{\mu_{lower}}{10\mu_{upper}}} \}$,
the iterate sequence generated by the Federated AGMs with partial device participation satisfies
 \begin{align*}
    \frac{1}{T}\sum_{t=0}^{T-1} \|\nabla f(x_{t})\|^2 &\leq  O( \sqrt{\frac{\mu_{upper} }{\mu_{lower}^3 \eta K T}} + \frac{K\sigma_g^2}{T} 
     +  (1+\frac{1}{S})\sqrt{\frac{\mu_{upper}^3 \eta^3 K}{\mu_{lower} T}}).
\end{align*} 
 \end{theorem}
 
Theorem \ref{th:p2} shows the convergence rate of  Federated AGMs, which is highly related to three items: the total number of iterations $TK$, the calibration parameters $\epsilon, p, \beta$ used in $\mu_{upper}$, and the gradient dissimilarity bound $\sigma_g^2$. 
Particularly, when increasing the number of inner loops $K$, the effects of dissimilarity among clients (non-IIDness) will be enlarged, which corresponds to severer client drift (\cite{karimireddy2019scaffold}). Our theoretical analysis from an angle shows the relationship between $K$ and client drift for nonconvex function (it has only been shown for quadratic function (\cite{charles2020outsized})).

For each calibration scheme, the $\mu_{lower}$ and $\mu_{upper}$ can be calculated explicitly as shown in the following table and then their convergence analysis can be further developed.

\begin{table}[h]
\centering
\caption{The $(\mu_{lower}, \mu_{upper})$ for each calibration scheme.}
\begin{tabular}{ c|c c c } 
 \toprule[0.5mm]
  Bounds & $\epsilon$-FedAdam & $p$-FedAdam & $s$-FedAdam \\ 
 \midrule
 $\mu_{lower}$ & $ O(\frac{1}{K \gamma_t \sqrt{1+\frac{1}{S}}})$ & $ O(\frac{1}{(K\gamma_t)^{2p} (1+\frac{1}{S})^p})$ & $ O(\frac{1}{K \gamma_t \sqrt{1+\frac{1}{S}}})$ \\ 
 \midrule
 $\mu_{upper}$ & $O(\frac{1}{\epsilon})$ & $O(\frac{1}{\epsilon^p})$ &  $ O(\beta)$ \\ 
\bottomrule[0.5mm]
\end{tabular}
\end{table}

\begin{corollary}[$\epsilon$- Federated AGMs]\label{cor:0}
If all the conditions in  Theorem \ref{th:p2}  hold, with appropriate $\eta = \Theta(\frac{1}{K })$ and $K \gamma_t < O( \min\{ \frac{1}{L},\sqrt[3]{\frac{\epsilon}{ \sqrt{1+\frac{1}{S}}}} \}) = O( \min\{ \frac{1}{L}, \epsilon^{\frac{1}{3}}\}) $,  we have 
$$ \min_{t=0,...,T-1} \|\nabla f(x_{t})\|^2 \leq  O(\sqrt{\frac{(1+1/S)^{3/2}}{\epsilon T}} + \frac{K\sigma_g^2}{T} + \sqrt{\frac{(1+1/S)^{5/2}}{\epsilon^3 K^2 T}} ).$$
\end{corollary}

\begin{corollary}[$p$- Federated AGMs]\label{cor:1}
If all the conditions in  Theorem \ref{th:p2}  hold, with appropriate $\eta = \Theta(\frac{1}{K })$ and $K \gamma_t <  O(\min\{ \frac{1}{L}, (\frac{\epsilon}{1+1/S})^{\frac{p}{2+2p}} \}) = O( \min\{ \frac{1}{L}, \epsilon^{\frac{p}{2+2p}}\})$, we have
$$ \min_{t=0,...,T-1} \|\nabla f(x_{t})\|^2 \leq    O( \sqrt{\frac{(1+1/S)^{3p}}{\epsilon^p T}} +\frac{K\sigma_g^2}{T}+ \sqrt{\frac{(1+1/S)^{2+p}}{\epsilon^{3p} K^2 T}}).$$
\end{corollary}

\begin{corollary}[$s$- Federated AGMs]\label{cor:2}
If all the conditions in  Theorem \ref{th:p2}  hold, with appropriate $\eta = \Theta(\frac{1}{K })$ and  $K \gamma_t < O( \min\{ \frac{1}{L},\sqrt[3]{\frac{1 }{\beta\sqrt{1+\frac{1}{S}} }} \}) = O( \min\{ \frac{1}{L}, \beta^{-\frac{1}{3}}\})$ , we have
$$ \min_{t=0,...,T-1} \|\nabla f(x_{t})\|^2 \leq O( \sqrt{\frac{\beta (1+1/S)^{3/2}}{T}}+ \frac{K\sigma_g^2}{T} + \sqrt{\frac{\beta^3 (1+1/S)^{5/2}}{K^2 T}}).$$
\end{corollary}
 
The alteration of inner step $K\gamma_t$ will be upper bounded by a constant related to calibration parameter. When $K$ is large, $\gamma_t$ should be correspondingly small; if we limit $K$, $\gamma_t$ may be larger. 
When $S$ is larger (more devices participate), the variance caused by the sampling of devices becomes smaller, and then the local stepsize $\gamma_t$ can be properly enlarged. 

As discussed in Methods section and to be consistent with later experimental observation, the calibration parameters do play important roles in the algorithm convergence to a stationary point. The convergence rate can be reduced by calibration as increasing $\epsilon <1$, decreasing $p$, and decreasing $\beta$. 
The above derivations are all based on the situation where $p_i$'s are quite different.  If $p_i$'s are the same $p_i = \frac{1}{N}$, our result is also applicable to the case with balanced on-device data.

\section{Experiments}
We compare our proposed methods: FedAdam, FedAMSGrad, and calibrated versions $\epsilon$/p/s-FedAdam,  $\epsilon$/p/s-FedAMSGrad, against several state-of-the-art FL methods, including FedAvg (\cite{mcmahan2016communication}), FedProx (\cite{li2018federated}), FedMomentum (\cite{hsu2019measuring}), SCAFFOLD (\cite{karimireddy2019scaffold}),  FedYogi (without first momentum in our setting) (\cite{reddi2020adaptive}).Notice that FedAdam is also included in (\cite{reddi2020adaptive}). More results are in the Appendix.

\subsection{Experimental Setup.}
We use three datasets for image classifications: MNIST, CIFAR10 and CIFAR100,  they are individually  tested on a CNN with 5 hidden layers, Residual Neural Network with 20 layers (ResNets 20) (\cite{he2016deep}) and VGGNet (\cite{simonyan2014very}).
During the training, we use a weight decay factor of $ 10^{-3}$ and a batch size of 64. For MNIST, we use LR decay with Reduce on Plateau scheme. For the CIFAR tasks, we use a fixed multi-stage LR decaying scheme: $\eta$ decays by $0.1$ at the $\frac{1}{2}$ total epochs and $\frac{3}{4}$ total epochs.  All algorithms perform grid search for hyper-parameters to choose from $\{10, 1, 0.1, 0.01, 0.001, 0.0001\}$ for $\eta$, $\{0.9, 0.99\}$ for $\beta_1$ and $\{0.99, 0.999\}$ for $\beta_2$.  For algorithm-specific hyper-parameters, they are tuned with the following criteria: $\mu \in \{0.1, 0.01, 0.001\}$ in FedProx, $\epsilon \in \{10^{-6},10^{-4}, 10^{-2}\}$, in $\epsilon$-FedAdam/FedAMSGrad, $p \in \{ \frac{1}{4}, \frac{1}{8}, \frac{1}{16}\}$ in $p$-FedAdam/FedAMSGrad, $\beta \in \{10,50,100, 200\}$ in $s$-FedAdam/FedAMSGrad.

\begin{table}[h]
\centering
\caption{Comparison of different methods on the MNIST data with participated clients S= 30.}\label{tb:MNIST}
\begin{tabular}{c |  c c c c }
\toprule[0.5mm]
 Method &K=10& K=50 & K= 100    \\
 \hline
FedAvg & $97.22\pm 0.23$  &  $97.67 \pm 0.26$& $97.56 \pm 0.10$  \\
FedProx & $97.21\pm 0.07$  &  $97.78 \pm 0.09$& $97.66 \pm 0.16$   \\
  FedMomentum&$96.79 \pm 0.42$ & $97.74 \pm 0.15$  & $97.57 \pm 0.18$    \\
 FedYogi&  $97.80\pm 0.18$ & $97.52 \pm 0.58$  & $96.80\pm 0.68$    \\
 SCAFFOLD&  $97.80\pm 0.53$ & $97.81 \pm 0.18$  & $97.60\pm 0.25$    \\
 \toprule[0.5mm]
  FedAdam &$98.57\pm 0.26$& $98.10\pm 0.11$ & $97.59\pm 0.11$  \\
  FedAMSGrad & $98.55\pm0.16$ & $98.15\pm0.18$ & $97.72\pm0.14$  \\
   $\epsilon$-FedAdam &  $\textbf{98.63}\pm \textbf{0.04}$ & $98.1\pm0.11$ & $97.66\pm0.23$   \\ 
  $\epsilon$-FedAMSGrad &  $98.55\pm0.16$ & $98.15\pm0.18$ & $97.72\pm0.14$  \\
  $p$-FedAdam  &  $98.41\pm0.14$ & $98.17\pm0.19$ & $97.72\pm0.06$   \\
  $p$-FedAMSGrad & $98.42\pm0.12$ & $98.05\pm0.14$ & $\textbf{97.76}\pm\textbf{0.12}$   \\
  $s$-FedAdam&  $98.48\pm0.06$ & $98.13\pm0.04$ & $97.53\pm0.21$   \\
  $s$-FedAMSGrad& $98.54\pm0.13$ & $\textbf{98.19}\pm\textbf{0.07}$ & $\textbf{97.76}\pm\textbf{0.13}$ \\
  \toprule[0.5mm]
\end{tabular}
\end{table}

\subsection{MNIST.}
As a sanity check, the MNIST dataset is used in our experiments where data decentralization is created with the sort-and-partition procedure (SP). Each device has data for two digits. Results are present in Table \ref{tb:MNIST}, showing that the FedAdam and FedAMSGrad can improve test accuracy in all settings. As expected, the test accuracy is further improved by the proposed calibrated versions.

\begin{figure}[h]
  \centering
  \includegraphics[trim=80 0 80 0,clip,width=\linewidth]{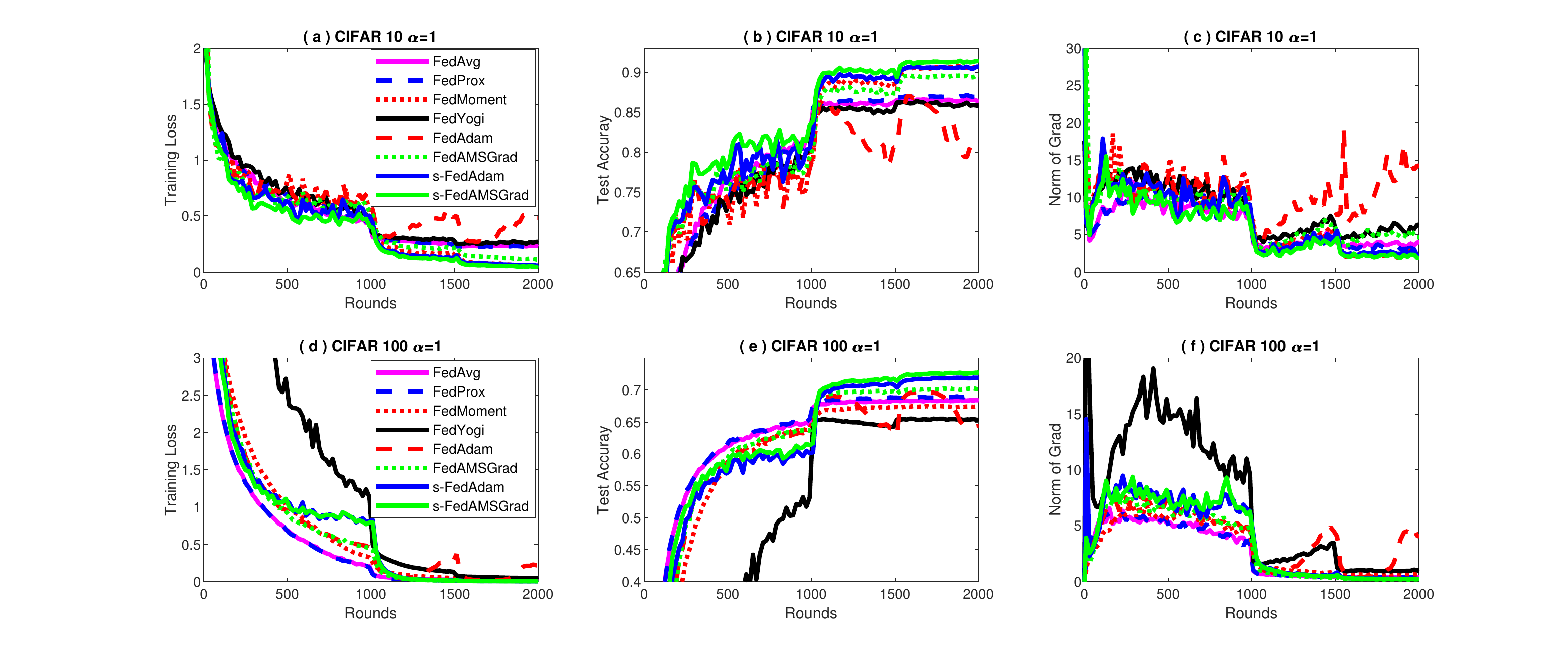}
  \caption{The comparison of different algorithms in terms of training loss (a,d,g), and testing accuracy (b,e,h), and norm of gradient (c,f,i) vs. communication rounds. The number of participated clients S= 10 and the number of local updates K = 10   for each client.  }
  \label{fig:sub3}
\end{figure}

\subsection{CIFAR10.} 
Using the PyTorch framework, we run the ResNets 20 model on CIFAR10 and results are shown in Figure \ref{fig:sub3}, \ref{fig:sub4}. Similar to (\cite{hsu2019measuring}), the federated CIFAR10 data is generated as: the number of clients is set to be 100 and the number of data points located on each client is set to be 500 for training data and 100 for testing data. 
 For each client, we use Dirichlet distribution to generate non-IID data.
Particularly, we observe that more computations are needed for smaller  $\alpha$ values, which means that clients are more dissimilar to each other. It is also safe to conclude that FL algorithms, with large dissimilarity across clients, tend to perform worse; however, calibrated Federated AGMs can always improve the learning performance.

\begin{figure}[t!]
  \centering
  \includegraphics[trim=50 10 40 0,clip,width=0.8\linewidth]{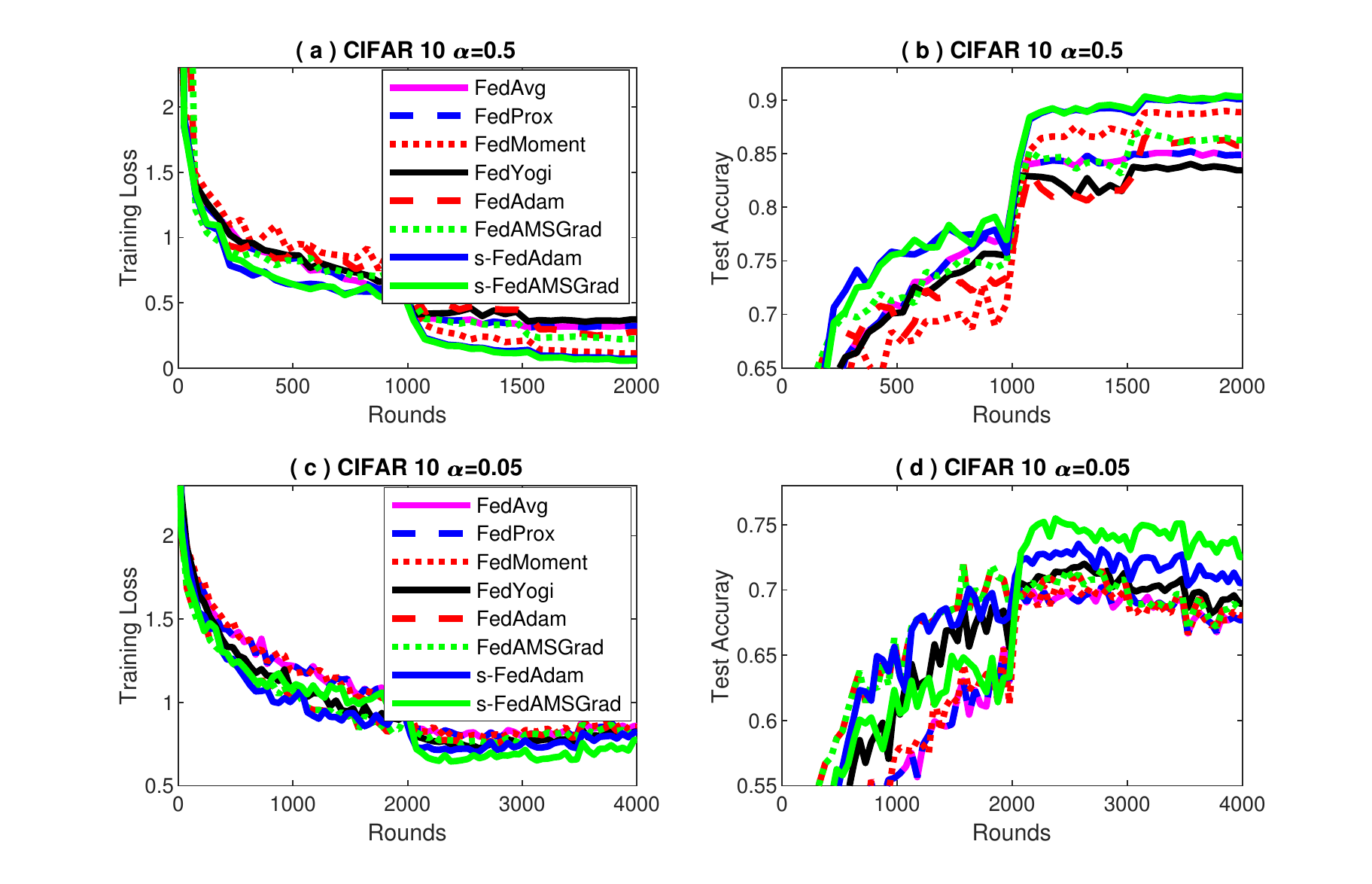}
  \caption{The comparison of algorithms under different non-IID level, $\alpha = \{0.5, 0.05\}$. The setting is same as in Figure \ref{fig:sub3}.}
  \label{fig:sub4}
  \vspace{-0.4cm}
\end{figure}

\begin{figure}[t]
  \centering
  \includegraphics[trim=50 0 40 0,clip,width=0.7\linewidth]{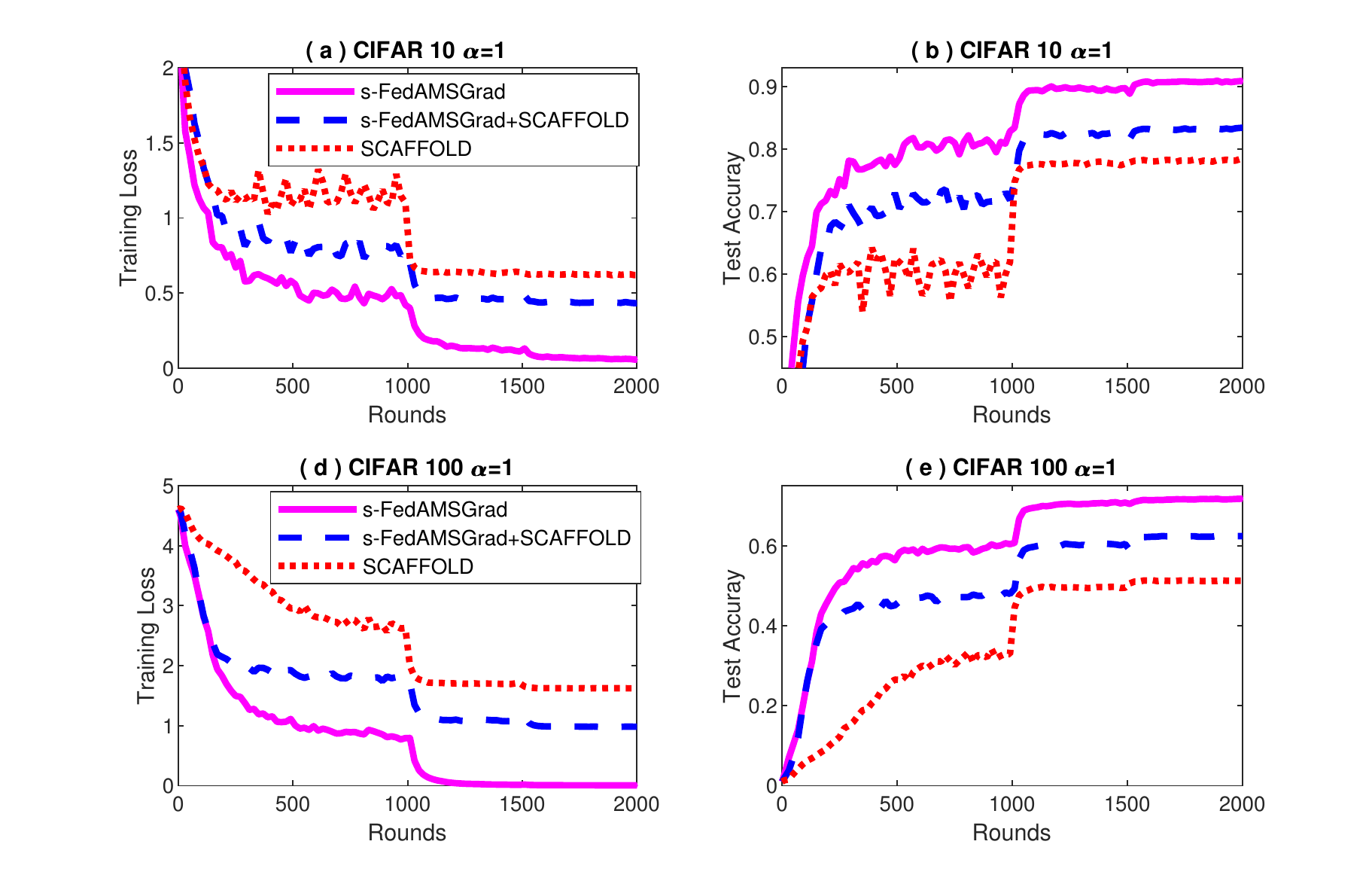}
  \caption{The comparison of SCAFFOLD, s-FedAMSGrad, and s-FedAMSGrad with SCAFFOLD under CIFAR10 and CIFAR100.  }
  \label{fig:sub5}
\end{figure}

\subsection{CIFAR100.}
The popular architecture VGGNet is also tested on CIFAR100 dataset to compare all algorithms. The federated CIFAR100 data is generated as: first split 100 labels into 10 large groups; then, we utilize Dirichlet distribution on top of 10 large groups to generate non-IID data. We set the number of participated clients $S= 10$, the number of local updates $K = 10$  for each client in Figure \ref{fig:sub5}, and more results about $S = 30$ are in appendix. The $s$-FedAMSGrad consistently achieves the highest test accuracy. SCAFFOLD tends to degrade in training large-scale deep learning applications. It may be due to the local client control variate in SCAFFOLD, which lags behind when only a small set of clients participate in FL at each round. However, when the $s$-FedAMSGrad with SVRG updates in inner loops (as done in SCAFFOLD) is used, the issue is alleviated.

In summary, we observe that Federated AGMs can improve test performance over existing federated learning methods, calibration further helps Federated AGMs, and stage-wise training (with exponential decay of the local learning rate $\gamma_t$) is also useful for federated learning of deep neural networks. Compared with calibrated Federated AGMs, the FedAvg suffer more from client drift, and SCAFFOLD tends to more easily trap in local minimals.

\section{Conclusion}
In this paper,  we propose a family of federated versions of adaptive gradient methods where clients run multiple steps of stochastic gradient descent updates before communication to a central server. Different calibration methods previously proposed for adaptive gradient methods are discussed and compared, providing insights into how the adaptive stepsize works in these methods. Our theoretical analysis shows that the convergence of the algorithms may rely on not only the inner loops, participating clients, gradient dissimilarity, but also to the calibration parameters. Empirically, we show that the proposed federated adaptive methods with careful calibration of the adaptive stepsize can converge faster and are less prone to trap into bad local minimizers in FL. 

\newpage
\section*{Appendix for Effective Federated Adaptive Gradient Methods with Non-IID Decentralized Data}
\renewcommand{\thesection}{\Alph{section}.\arabic{section}}
\setcounter{section}{0}

\section{Architecture Used in Our Experiments}

Here we mainly introduce the MNIST architecture with Pytorch used in our empirical study, ResNets and DenseNets are well-known architectures used in many works and we do not include details here.

\begin{center}
 \begin{tabular}{||c | c ||} 
 \hline
 layer & layer setting \\ [0.5ex] 
 \hline\hline
 F.relu(self.conv1(x)) & self.conv1 = nn.Conv2d(1, 6, 5)\\ 
 \hline
 F.max\_pool2d(x, 2, 2) & \\
 \hline
 F.relu(self.conv2(x))  &self.conv2 = nn.Conv2d(6, 16, 5) \\
 \hline
 x.view(-1, 16*4)  & \\
 \hline
 F.relu(self.fc1(x)) & self.fc1 = nn.Linear(16*4*4, 120)\\ 
 \hline
 x= F.relu(self.fc2(x)) &self.fc2 = nn.Linear(120, 84) \\ 
 \hline
  x = self.fc3(x)& self.fc3 = nn.Linear(84, 10)\\
  \hline
  F.log\_softmax(x, dim=1) & \\
 \hline
\end{tabular}
\end{center}

\section{More Empirical Results}

\subsection{Non-IIDness}

As introduced in the main paper, $\alpha$ is concentration parameter for Dirichlet distribution to control the degree of dissimilarity across devices. With $\alpha \rightarrow 0$, we can generate real different distributions among clients, with $\alpha \rightarrow \infty$, all clients have identical distributions to the prior. 

\begin{figure}[H]
\centering
  \centering
  \includegraphics[width=0.9\linewidth]{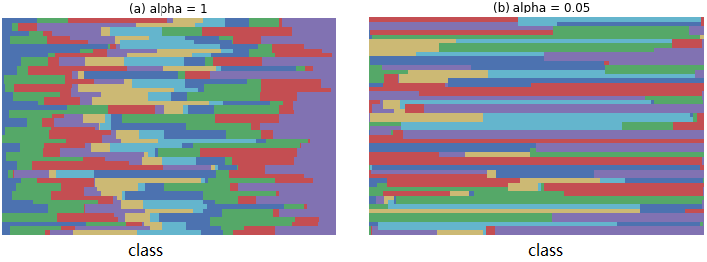}
  \caption{class distribution generated by Dirichlet distribution with different $\alpha$ for CIFAR10 dataset}
  \label{fig:dirichlet}
\end{figure}

\subsection{CIFAR 10}

\begin{figure}[H]
\centering
  \centering
  \includegraphics[trim=80 0 80 0,clip,width=\textwidth]{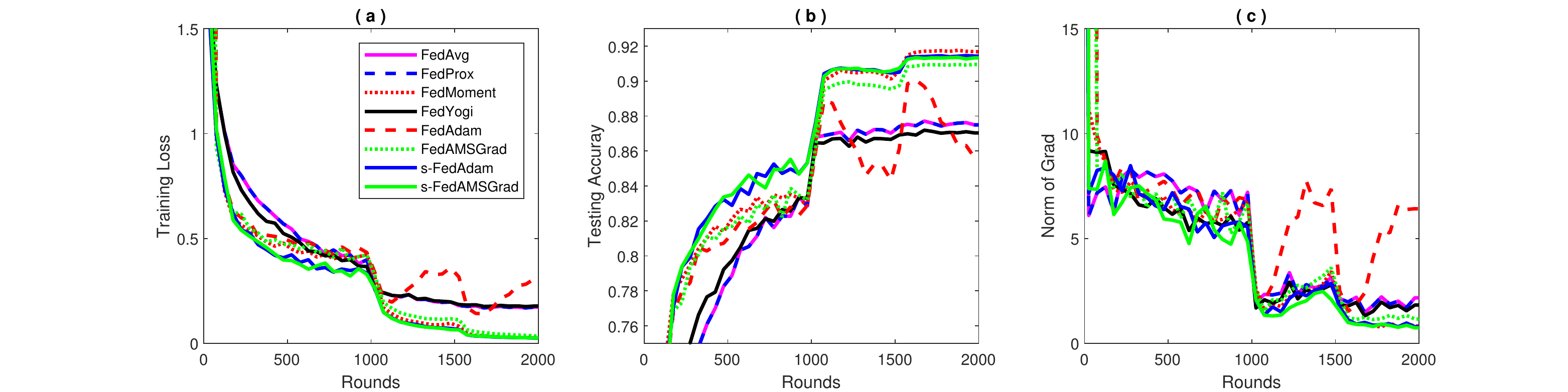}
 \caption{Comparison of all methods for CIFAR10 data with participated clients S= 10,  1 epoch for inner loop  and concentration parameter $\alpha = 100$. }
 \label{fig:2}
\end{figure}

We first observe how the identicalness affects experimental performance.
Figure \ref{fig:2} shows the performance of all the FL methods for CIFAR10 data with participated clients $S=10$, and concentration parameter $\alpha = 100$.
Figure \ref{fig:3} and Figure \ref{fig:4} use concentration parameter $\alpha = 1$ for calibrated FedAdam and calibrated FedAMSGrad, respectively.;
Figure \ref{fig:5} and Figure \ref{fig:6} use concentration parameter $\alpha = 0.05$ for calibrated FedAdam and calibrated FedAMSGrad, respectively.
Notice that with the increase of concentration parameter $\alpha$, the data located on local clients are more identical, and the performance will be better.

Then we provide comparisons of calibrated Federated AGMs with multi-stage LR decay scheme for ResNets with CIFAR10 dataset. 
Figure \ref{fig:3} and Figure \ref{fig:5} are comparisons of  $\epsilon$-FedAdam, $p$-FedAdam and $s$-FedAdam with multi-stage LR decay scheme for ResNets with CIFAR10 dataset; Figure \ref{fig:4} and Figure \ref{fig:6} are comparisons of $\epsilon$-FedAMSGrad, $p$-FedAMSGrad and $s$-FedAMSGrad with multi-stage LR decay scheme for ResNets with CIFAR10 dataset. We can see that no matter FedAdam or FedAMSGrad, calibration techniques posed on second-order momentum matter lot. With appropriate calibrating parameters, Federated AGMs can always improve the training and testing performance a lot.

\begin{figure}[H]
\centering
  \centering
  \includegraphics[trim=80 0 80 0,clip,width=\linewidth]{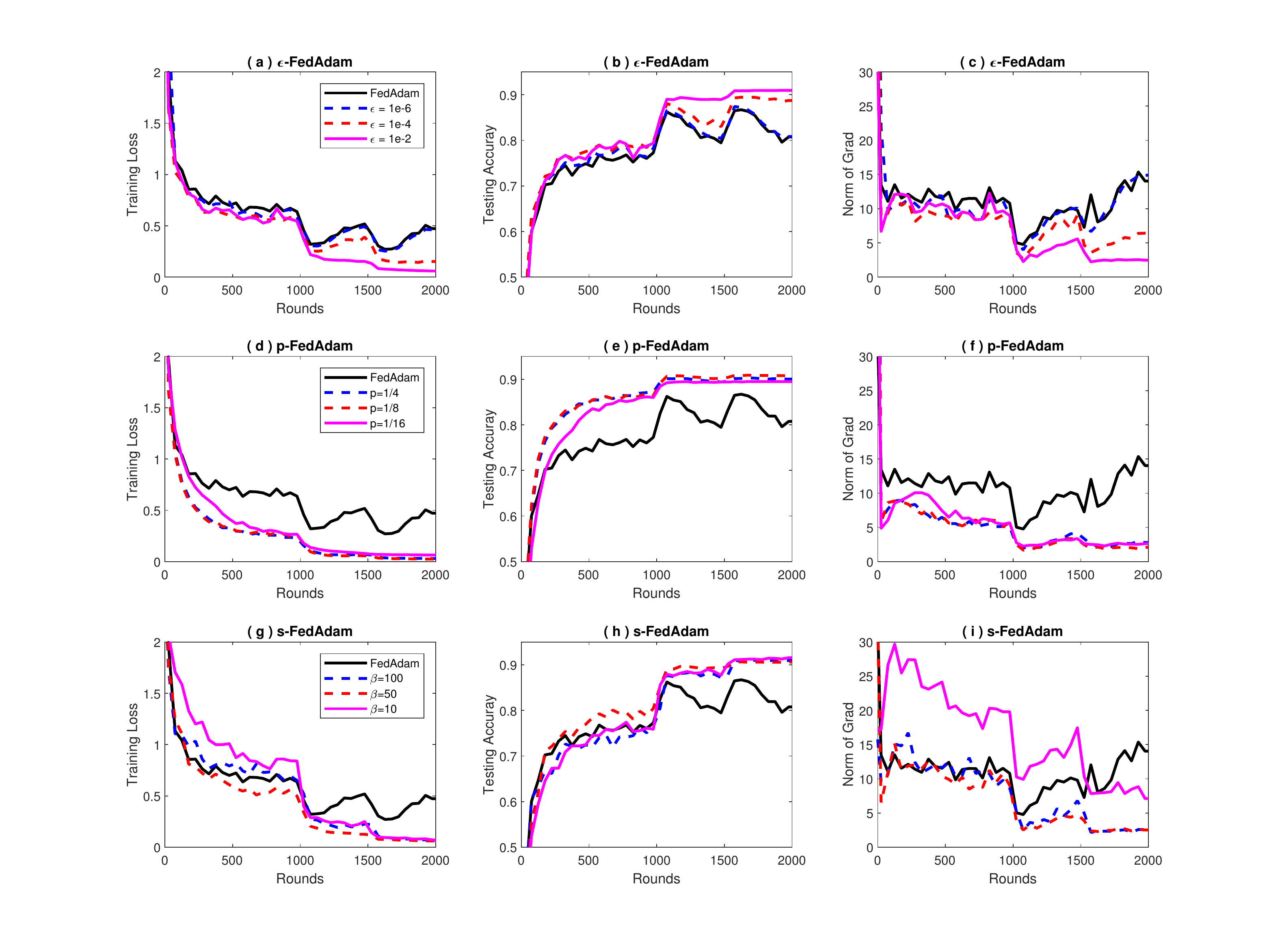}
  \caption{The training loss (a, d, g), testing accurcay (b, e, h) and norm of gradient (c, f, i) for $\epsilon$-FedAdam, $p$-FedAdam and $s$-FedAdam with multi-stage LR decay scheme for ResNets with CIFAR10 dataset, participated clients S= 10,  1 epoch for inner loop   and concentration parameter $\alpha = 1$.}
  \label{fig:3}
\end{figure}

\begin{figure}[H]
\centering
  \centering
  \includegraphics[trim=80 0 80 0,clip,width=\linewidth]{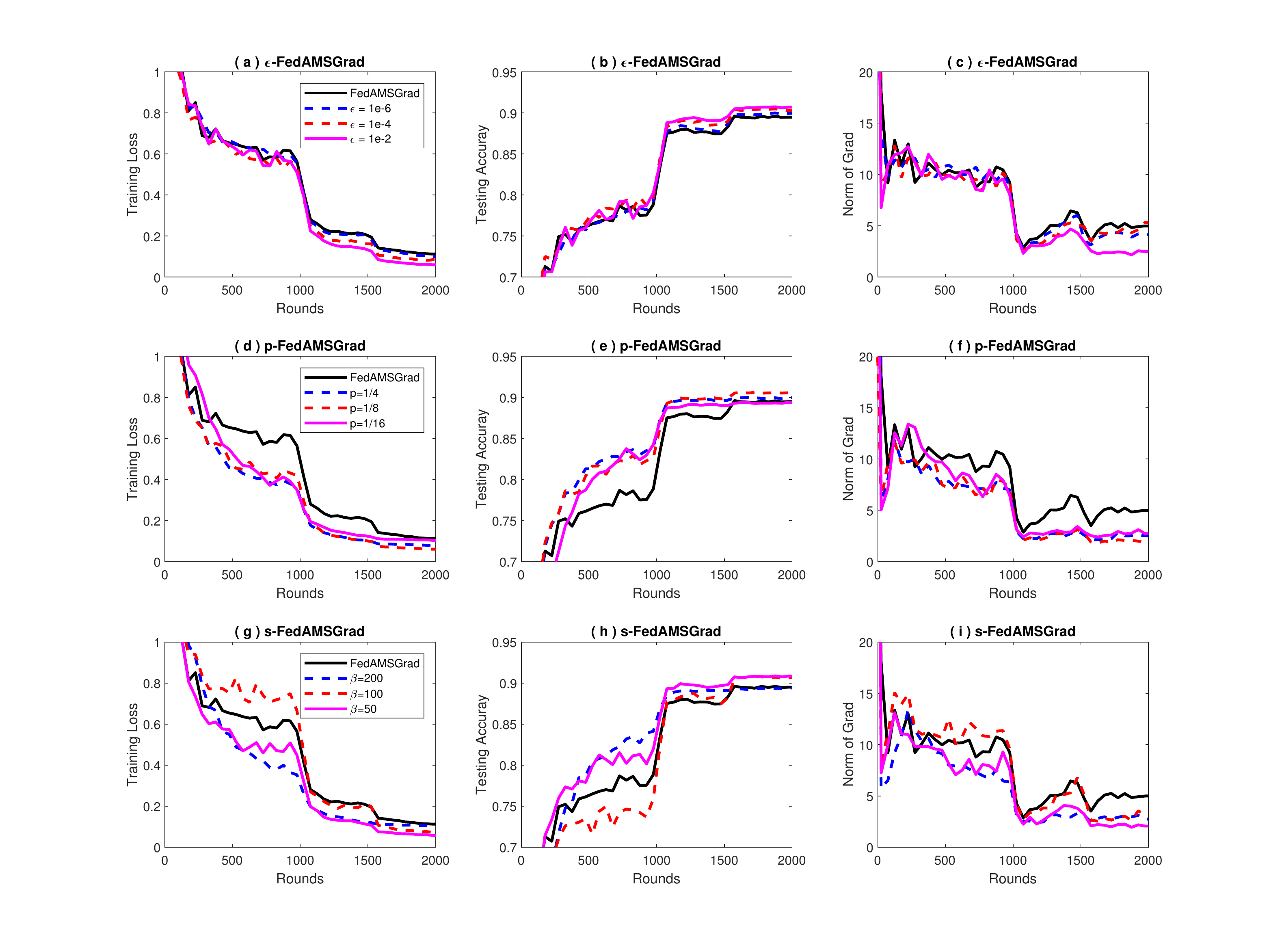}
  \caption{The training loss (a, d, g), testing accurcay (b, e, h) and norm of gradient (c, f, i) for $\epsilon$-FedAMSGrad, $p$-FedAMSGrad and $s$-FedAMSGrad with multi-stage LR decay scheme for ResNets with CIFAR10 dataset, participated clients S= 10,  1 epoch for inner loop   and concentration parameter $\alpha = 1$.}
  \label{fig:4}
\end{figure}

\begin{figure}[H]
\centering
  \centering
  \includegraphics[trim=80 0 80 0,clip,width=\textwidth]{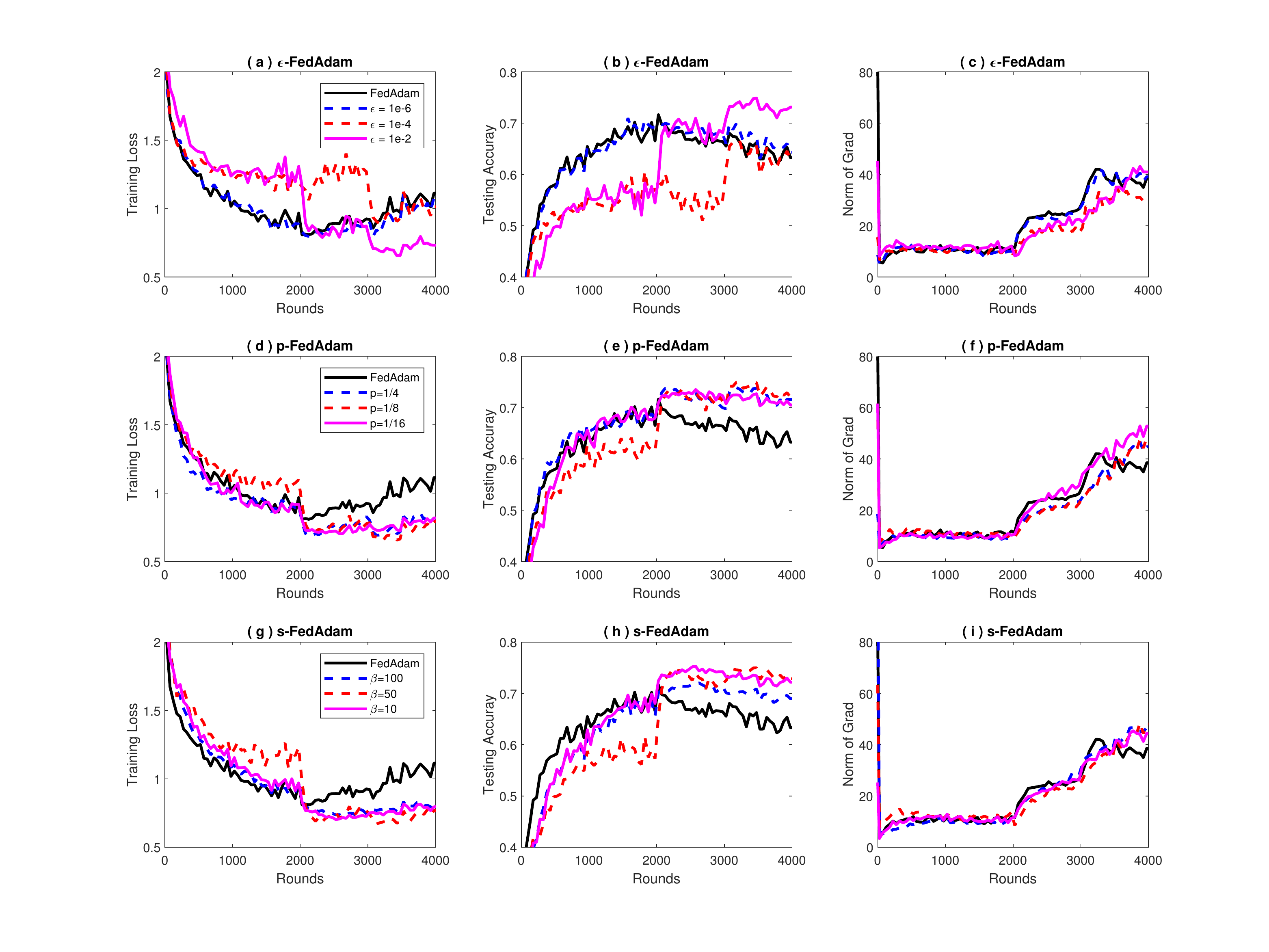}
  \caption{The training loss (a, d, g), testing accurcay (b, e, h) and norm of gradient (c, f, i) for $\epsilon$-FedAdam, $p$-FedAdam and $s$-FedAdam with multi-stage LR decay scheme for ResNets with CIFAR10 dataset, participated clients S= 10,  1 epoch for inner loop  and concentration parameter $\alpha = 0.05$.}
  \label{fig:5}
\end{figure}

\begin{figure}[H]
\centering
  \centering
  \includegraphics[trim=80 0 80 0,clip,width=0.95\textwidth]{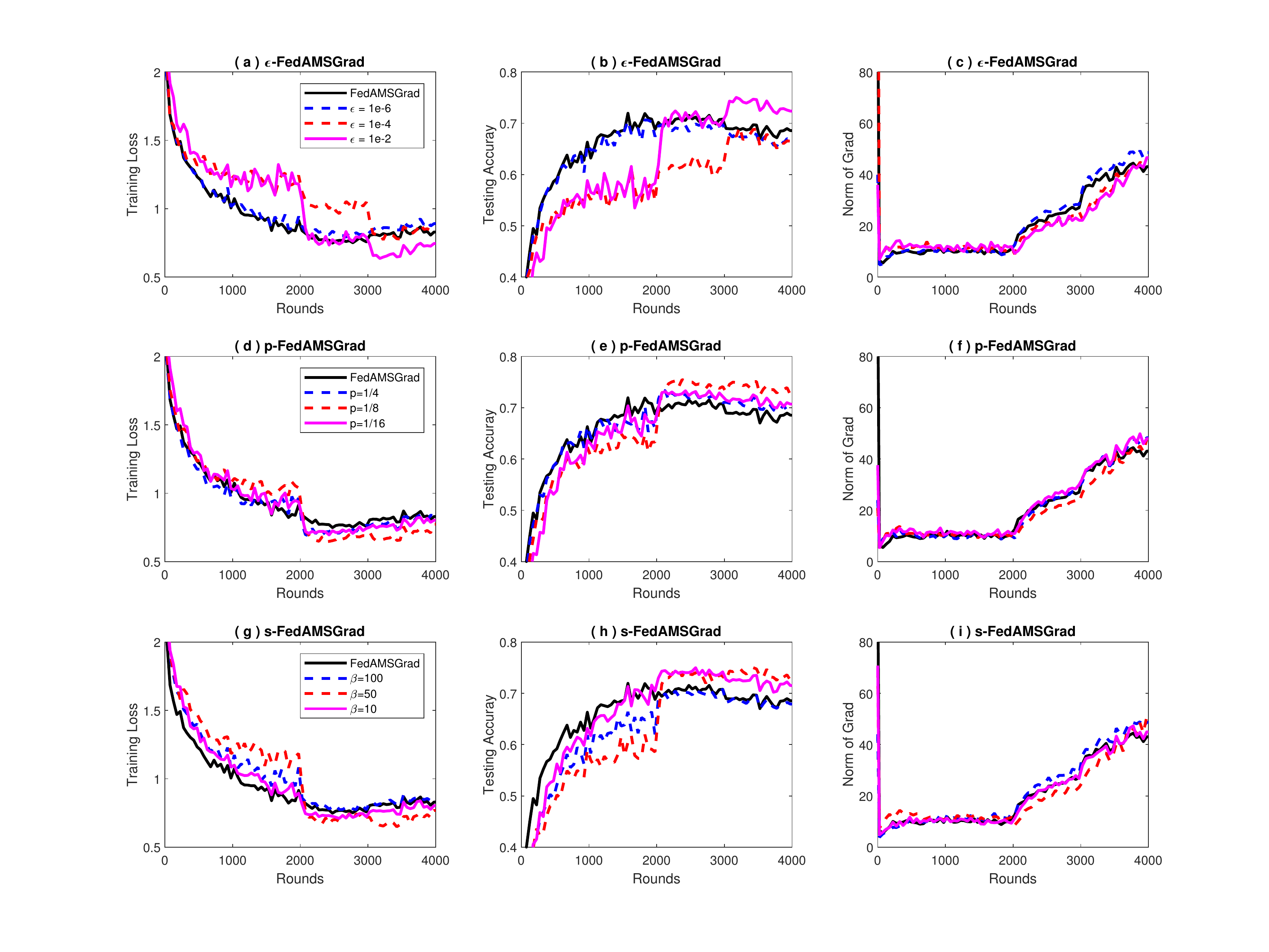}
  \vspace{-0.5cm}
  \caption{The training loss (a, d, g), testing accurcay (b, e, h) and norm of gradient (c, f, i) for $\epsilon$-FedAMSGrad, $p$-FedAMSGrad and $s$-FedAMSGrad with multi-stage LR decay scheme for ResNets with CIFAR10 dataset, participated clients S= 10,  1 epoch for inner loop   and concentration parameter $\alpha = 0.05$.}
  \label{fig:6}
\end{figure}

\subsection{CIFAR 100}
\begin{figure}[H]
\centering
  \centering
  \includegraphics[trim=80 0 80 0,clip,width=0.95\linewidth]{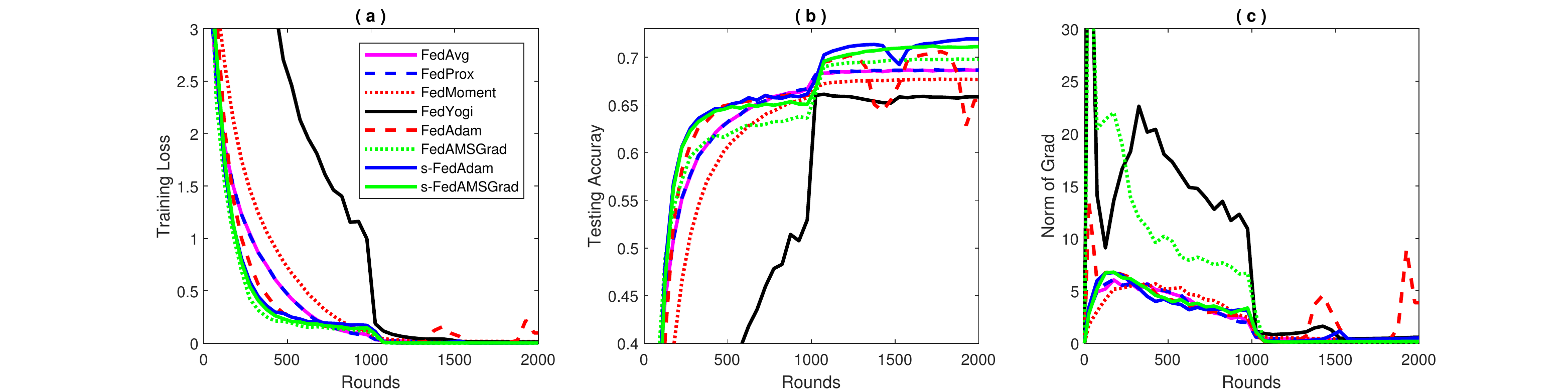}
  \vspace{-0.3cm}
  \caption{Comparison of all methods for CIFAR100 data with participated clients S= 30,  1 epoch for inner loop   and concentration parameter $\alpha = 1$}
  \label{fig:7}
\end{figure}

We also provide more comparisons of calibrated Federated AGMs with multi-stage LR decay scheme for VGGNet with CIFAR100 dataset (see Figure \ref{fig:7}).  And Figure \ref{fig:8} is a comparison among  s-FedAMSGrad, s-FedAMSGrad+SFAFFOLD and SCAFFOLD methods under different non-IIDness ($\alpha \in \{1, 0.5, 0.05\}$) in CIFAR10 and CIFAR100.

\begin{figure}[H]
\centering
  \centering
  \includegraphics[trim=80 0 80 0,clip,width=\linewidth]{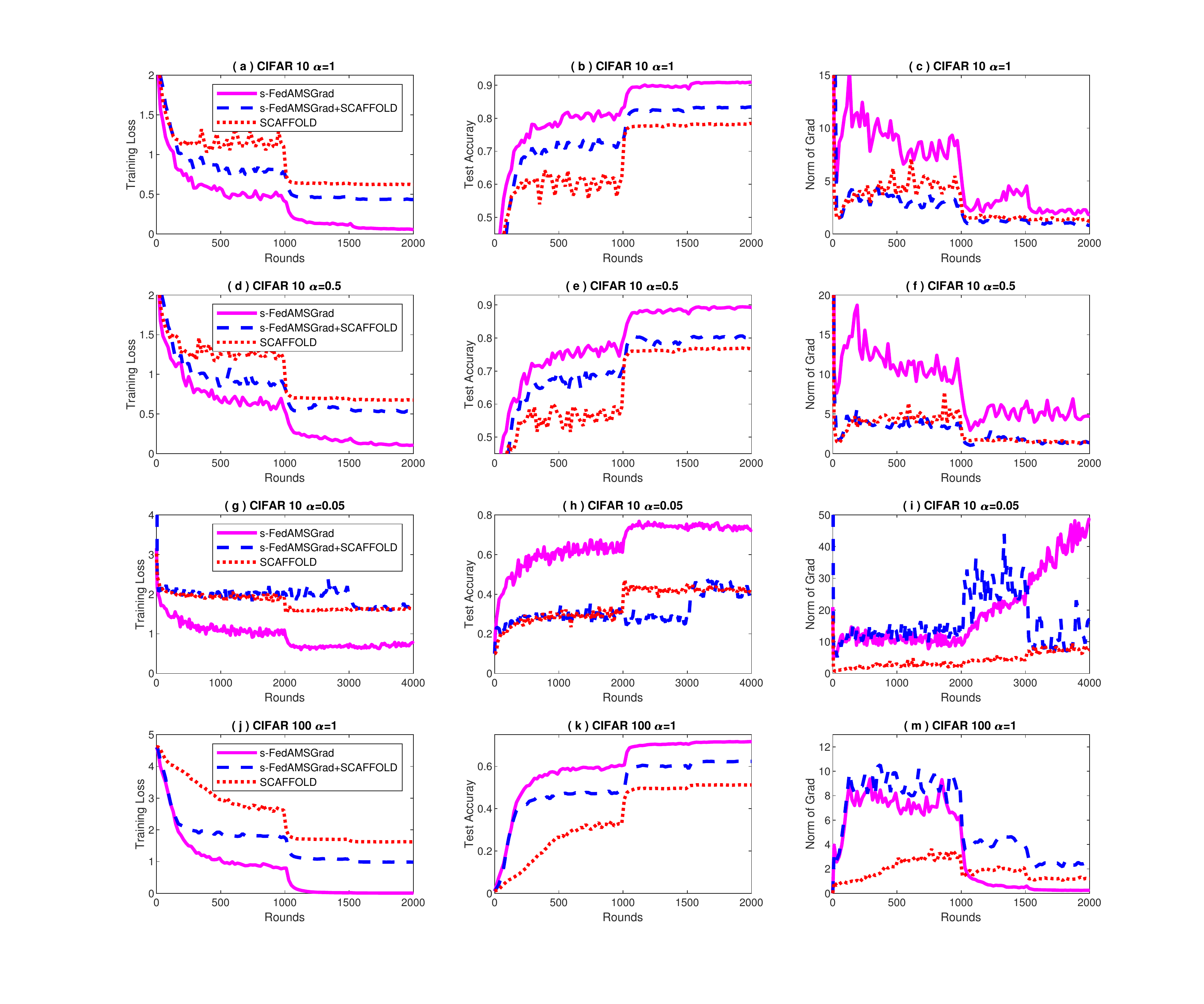}
  \caption{Comparison among s-FedAMSGrad, s-FedAMSGrad+SFAFFOLD and SCAFFOLD.}
  \label{fig:8}
\end{figure}

\newpage

\section{Theoretical Analysis Details}
We provide theoretical analysis of our proposed Federated AGMs and especially with different calibration techniques. Different with the previous analysis of FedAdam in \cite{reddi2020adaptive}, we include the momentum analysis into the whole theoretical part, and the effect of calibration parameters play an essential role in the overall convergence performance.


\subsection{Prepared Lemmas}
Let's define first with the stochastic gradient under partial device participation,
\begin{equation*}
    g_{t,k}=\frac{1}{S}\sum_{i\in S_t}  g_{t,k}^{(i)};
\end{equation*}
and the virtual direction will be 
\begin{align*}
    \Delta_t =  x_{t} - \tilde{x}_{t+1}
     =   \frac{1}{S}\sum_{i\in S_{t}} (x_{t} - x_{t,K}^{(i)})
     =  \frac{1}{S}\sum_{i\in S_{t}} (x_{t,0}^{(i)} - x_{t,K}^{(i)})
     = \frac{1}{S}\sum_{i\in S_{t}}\sum_{k=0}^{K-1}  \gamma_t g_{t,k}^{(i)} = \sum_{k=0}^{K-1}  \gamma_t g_{t,k};
\end{align*}
With full device participation, we can accordingly have,
\begin{equation*}
    \tilde{g}_{t,k}=\sum_{i=1}^{N} p_i g_{t,k}^{(i)};    
    ~~~~\Delta_t  = \sum_{k=0}^{K-1}  \gamma_t \tilde{g}_{t,k}.
\end{equation*}
In the following analysis, we always consider the partial device participated case.

We have a series of prepared lemmas to help with optimization convergence rate analysis.
 
 \begin{lemma}\label{g_t_full}
Assume the above assumptions hold, we can easily derive the properties of unbiased stochastic gradient with full device participation,
\begin{equation*}
    E[\tilde{g}_{t,k}]=\sum_{i=1}^{N} p_i \nabla f_i(x_{t,k}^{(i)});
\end{equation*}
\begin{equation*}
    E[\| \tilde{g}_{t,k} -\sum_{i=1}^{N} p_i\nabla f_i(x_{t,k}^{(i)})\|^2]\leq  \sum_{i=1}^N p_i \sigma_i^2;
\end{equation*}
\begin{equation*}
	E[\|\tilde{g}_{t,k}\|^2]\leq 2\sum_{i=1}^N p_i (\sigma_i^2 + G_i^2).
\end{equation*}
\end{lemma}

\begin{proof}
From the problem formulation,
    \begin{equation*}
    E[\tilde{g}_{t,k}]= E[\sum_{i=1}^N p_i g_{t,k}^{(i)}]
    =\sum_{i=1}^N p_i E[g_{t,k}^{(i)}]
    =\sum_{i=1}^N p_i \nabla f_i(x_{t,k}^{(i)}).
    \end{equation*}
    \begin{align*}
    E[\| \tilde{g}_{t,k} -\sum_{i=1}^N p_i \nabla f_i(x_{t,k}^{(i)})\|^2]&= E[\|\sum_{i=1}^N p_i g_{t,k}^{(i)}-\sum_{i=1}^N p_i\nabla f_i(x_{t,k}^{(i)})\|^2]\\
    &\leq\sum_{i=1}^N p_i E[\| g_{t,k}^{(i)}-\nabla f_i(x_{t,k}^{(i)})\|^2]
    \leq \sum_{i=1}^N p_i \sigma_i^2.
    \end{align*}
    The inequality holds due to Jensen's inequality.
    \begin{align*}
	E[\|\tilde{g}_{t,k}\|^2]& =E[\|\tilde{g}_{t,k} - \sum_{i=1}^N p_i \nabla f_i(x_{t,k}^{(i)}) +\sum_{i=1}^N p_i \nabla f_i(x_{t,k}^{(i)})\|^2]\\
	&\leq 2 E[\| \tilde{g}_{t,k} -\sum_{i=1}^N p_i \nabla f_i(x_{t,k}^{(i)})\|^2] + 2E[\|\sum_{i=1}^N p_i \nabla f_i(x_{t,k}^{(i)})\|^2] \\
	&\leq  2\sum_{i=1}^N p_i  \sigma_i^2 + 2E[\|\sum_{i=1}^N p_i\nabla f_i(x_{t,k}^{(i)})\|^2]\\
	&\leq 2\sum_{i=1}^N p_i  \sigma_i^2 + 2\sum_{i=1}^N p_i G_i^2.
	\end{align*}
\end{proof}   
 
\begin{lemma}\label{g_t_2}
Assume the above assumptions hold, we can easily derive the properties of unbiased stochastic gradient with partial device participation,
\begin{equation*}
    E[g_{t,k}]=\sum_{i=1}^{N} p_i \nabla f_i (x_{t,k}^{(i)});
\end{equation*}
\begin{equation*}
	E[\|g_{t,k}\|^2]\leq \underbrace{   \frac{1}{S} (12 \sum_{i=1}^N p_i \sigma_i^2 + 24    \sum_{i=1}^{N} p_i G_i^2)}_{partial~ participation}
	+ \underbrace{4\sum_{i=1}^N p_i  (\sigma_i^2 + G_i^2).}_{local~update}
\end{equation*}

With all device active, the partial participation term can be removed.
\end{lemma}

\begin{proof}
From the problem formulation, $ g_{t,k}=  \frac{1}{S}\sum_{i\in S_{t}}  g_{t,k}^{(i)}$,
    \begin{align*}
    E[g_{t,k}]&= E[\frac{1}{S}\sum_{i\in S_{t}} g_{t,k}^{(i)}]
    =\frac{1}{S} E[\sum_{i\in S_{t}}   g_{t,k}^{(i)}] 
    =\frac{1}{S} E^{z}[\sum_{i\in S_{t}}   E^{(i)} [g_{t,k}^{(i)}] ]\\
    &=\frac{1}{S} \sum_{i\in S_{t}}   \sum_{i=1}^{N} p_i \nabla f_i (x_{t,k}^{(i)})
    = \sum_{i=1}^{N} p_i \nabla f_i (x_{t,k}^{(i)}).
    \end{align*}
Suppose $E^{z}$ is expectation over data sampling and $E^{(i)}$ is expectation over client sampling, then we have the following:    
    \begin{align*}
	&E[\|g_{t,k}\|^2] = E^{z}E^{(i)}[\|g_{t,k}\|^2] 
	=E^{z}E^{(i)}[\|\frac{1}{S} \sum_{i\in S_t} g_{t,k}^{(i)}\|^2]\\
	&=E^{z}E^{(i)}[\|\frac{1}{S} \sum_{i\in S_t} g_{t,k}^{(i)}-\sum_{i=1}^{N} p_i g_{t,k}^{(i)} +\sum_{i=1}^{N} p_i g_{t,k}^{(i)}\|^2]\\
	&\leq 2 E^{z}E^{(i)}[\|\frac{1}{S} \sum_{i\in S_t} g_{t,k}^{(i)}-\sum_{i=1}^{N} p_i g_{t,k}^{(i)}\|^2] + 2E^{z} [\| \sum_{i=1}^{N} p_i g_{t,k}^{(i)}\|^2] \\
	&\leq  2 E^{z} \frac{1}{S^2} \sum_{i\in S_t}E^{(i)}[\| g_{t,k}^{(i)}-\sum_{i=1}^{N} p_i g_{t,k}^{(i)}\|^2] + 2E^{z} [\| \sum_{i=1}^{N} p_i g_{t,k}^{(i)}\|^2]\\
	&\leq  2 E^{z} \frac{1}{S^2} \sum_{i\in S_t}E^{(i)}[\| g_{t,k}^{(i)}-\nabla f_i(x_{t,k}^{(i)}) + \nabla f_i(x_{t,k}^{(i)}) -\sum_{i=1}^N p_i \nabla f_i (x_{t,k}^{(i)}) +\sum_{i=1}^N p_i \nabla f_i (x_{t,k}^{(i)}) -\sum_{i=1}^{N} p_i g_{t,k}^{(i)}\|^2] \\
	&+ 2E^{z} [\| \sum_{i=1}^{N} p_i g_{t,k}^{(i)}\|^2]\\
	&\leq 6 E^{z} \frac{1}{S^2} \sum_{i\in S_t}E^{(i)}[\| g_{t,k}^{(i)}-\nabla f_i(x_{t,k}^{(i)})\|^2] 
	+ 6 E^{z}  \frac{1}{S^2} \sum_{i\in S_t}E^{(i)}[\| \nabla f_i(x_{t,k}^{(i)}) -\sum_{i=1}^{N} p_i \nabla f_i(x_{t,k}^{(i)})\|^2]\\
	&+ 6 E^{z}  \frac{1}{S^2} \sum_{i\in S_t}E^{(i)}[\|  \sum_{i=1}^{N} p_i \nabla f_i(x_{t,k}^{(i)}) - \sum_{i=1}^{N} p_i g_{t,k}^{(i)}\|^2]+ 2E^{z} [\| \sum_{i=1}^{N} p_i g_{t,k}^{(i)}\|^2]\\
	&\leq  6   \frac{1}{S} \sum_{i=1}^{N} p_i \sigma_i^2
	+ 6  E^z  \frac{1}{S^2} \sum_{i\in S_t} E^{(i)}(2\|\nabla f_i(x_{t,k}^{(i)})\|^2 + 2\|\sum_{i=1}^{N} p_i \nabla f_i(x_{t,k}^{(i)})\|^2)
	+6    \frac{1}{S}  \sum_{i=1}^{N} p_i \sigma_i^2 \\
	&+ 2E^{z} [\| \sum_{i=1}^{N} p_i g_{t,k}^{(i)}\|^2]\\
	&\leq  6  \frac{1}{S} \sum_{i=1}^N p_i \sigma_i^2 + 24   \frac{1}{S}   \sum_{i=1}^{N} p_i G_i^2
	+6    \frac{1}{S} \sum_{i=1}^{N} p_i \sigma_i^2 
	+ 2  (2\sum_{i=1}^N p_i  \sigma_i^2 + 2\sum_{i=1}^N p_i G_i^2)\\
	&=\underbrace{   \frac{1}{S} (12 \sum_{i=1}^N p_i \sigma_i^2 + 24    \sum_{i=1}^{N} p_i G_i^2)}_{partial~ participation}
	+ \underbrace{4\sum_{i=1}^N p_i  (\sigma_i^2 + G_i^2).}_{full~participation}
	\end{align*}
\end{proof}   

Now we are going to measure the virtual direction and the related slow momentum and second-order momentum. Notice that these vectors are calculated beyond each working nodes self iteration, then we only use subscript $t$ to denote the current iterate.

\begin{lemma}\label{delta_t}
	For virtual direction, we have
	\begin{equation}
	E[\|\Delta_{t}\|^2]\leq \mathcal{V},
	\end{equation}
	where $\mathcal{V} := \frac{K^2 \gamma_t^2}{S} (12 \sum_{i=1}^N p_i \sigma_i^2 + 24 \sum_{i=1}^{N} p_i G_i^2) 
	+ 4K^2 \gamma_t^2 \sum_{i=1}^N p_i  (\sigma_i^2 + G_i^2)$.
\end{lemma}

\begin{proof}
	\begin{align*}
	E[\|\Delta_{t}\|^2]& = 	E[\| \frac{1}{S}\sum_{i\in S_{t}}\sum_{k=0}^{K-1} g_{t,k}^{(i)}\|^2] 
	=E[\|\sum_{k=0}^{K-1} \gamma_t g_{t,k}\|^2]\\
	&\leq \frac{K^2 \gamma_t^2}{S} (12 \sum_{i=1}^N p_i \sigma_i^2 + 24 \sum_{i=1}^{N} p_i G_i^2) 
	+ 4K^2 \gamma_t^2\sum_{i=1}^N p_i  (\sigma_i^2 + G_i^2)
	= \mathcal{V}.
	\end{align*}
\end{proof}

\begin{remark}
In our analysis, $\mathcal{V}$ is a very important term that effects the convergence rate, and we notice that $\mathcal{V} = O(K^2 \gamma_t^2(1+\frac{1}{S})) $ when treated $p_i, \sigma_i, G_i$ as constants, then we know it is related with inner loop iterations, inner loop stepsize and partial device numbers, which is verified in our experiments.
\end{remark}

\begin{lemma}\label{m_t}
All momentum-based optimizers using first momentum $m_t=\beta_1 m_{t-1}+(1-\beta_1)\Delta_t$ will satisfy 
\begin{equation}
    E[\|m_t\|^2]\leq \mathcal{V}.
\end{equation}
\end{lemma}

\begin{proof}
From the updating rule of first momentum estimator, we can derive 
\begin{equation}
    m_t = \Sigma_{l=1}^t(1-\beta_1)\beta_1^{t-l}\Delta_l.
\end{equation}

Let $\Gamma_t=\Sigma_{l=1}^t\beta_1^{t-l}= \frac{1-\beta_1^t}{1-\beta_1}$, by Jensen's inequality and Lemma \ref{delta_t},
\begin{align*}
    E[\|m_t\|^2] &= E[\|\Sigma_{l=1}^t(1-\beta_1)\beta_1^{t-l}\Delta_l\|^2]
    = \Gamma_t^2 E[\|\Sigma_{l=1}^t\frac{(1-\beta_1)\beta_1^{t-l}}{\Gamma_t}\Delta_l\|^2] \\
    &\leq \Gamma_t^2 \Sigma_{l=1}^t(1-\beta_1)^2\frac{\beta_1^{t-l}}{\Gamma_t}E[\|\Delta_l\|^2] \leq \Gamma_t(1-\beta_1)^2\Sigma_{l=1}^t\beta_1^{t-l} \mathcal{V}\\
    &\leq  \mathcal{V}.
\end{align*}
\end{proof}

Now let's consider adaptive algorithms in Federated learning, besides the above Lemma \ref{g_t_2}, \ref{delta_t}, \ref{m_t}, we also need to bound the adaptive term as follows.

\begin{lemma}\label{v_t}
Each coordinate of vector $v_t=\beta_2 v_{t-1}+(1-\beta_2)\Delta_t^2$ will satisfy 
\begin{align*}
    E[v_{t,j}]\leq  \mathcal{V},
\end{align*}
where $j \in [d]$ is the coordinate index.
\end{lemma}

\begin{proof}
	From the updating rule of second momentum estimator, we can derive 
	\begin{equation}
	v_{t,j} = \Sigma_{l=1}^t(1-\beta_2)\beta_2^{t-l}\Delta_{l,j}^2 \geq 0.
	\end{equation}
	Since the decay parameter $\beta_2 \in [0,1)$, $\Sigma_{l=1}^t(1-\beta_2)\beta_2^{t-l}=1-\beta_2^t \leq 1$. From Lemma \ref{delta_t},
	\begin{align*}
	E[v_{t,j}]=E[\Sigma_{l=1}^t(1-\beta_2)\beta_2^{t-l}\Delta_{l,j}^2]
	\leq \Sigma_{l=1}^t(1-\beta_2)\beta_2^{t-l}\mathcal{V}
	\leq  \mathcal{V}.
	\end{align*}
\end{proof}

For the sake of simplicity, we combine the analysis of $\epsilon$-FedAdam and $p$-FedAdam methods together and use $p$-FedAdam's adaptive learning rate as a more general format, when $p= \frac{1}{2}$, it degrades to the original FedAdam. And we can derive the following important lemma:

\begin{lemma}\label{bound1} 
For any $t\geq 1$, $j \in [d]$, $\beta_2\in [0,1]$, and $\epsilon$ in $p$-FedAdam methods (including $\epsilon$-FedAdam), $\beta$ in $s$-FedAdam, the following bounds always hold:

$\epsilon$-FedAdam has $(\mu_1,\mu_2)-$ bounded adaptive learning rate:
\begin{equation}  
    \mu_1 \leq \frac{1}{(v_{t,j})^{1/2}+\epsilon} \leq \mu_2;
\end{equation}

$p$-FedAdam has $(\mu_3,\mu_4)-$ bounded adaptive learning rate:
\begin{equation}  
    \mu_3 \leq \frac{1}{(v_{t,j}+\epsilon)^{p}} \leq \mu_4;
\end{equation}

$s$-FedAdam has $(\mu_5,\mu_6)-$ bounded adaptive learning rate:
\begin{equation}  
    \mu_5 \leq \frac{1}{softplus(\sqrt{v_{t,j}})} \leq \mu_6.
\end{equation}
For simplicity, we use pairs $(\mu_{lower}, \mu_{upper})$ to denote the calibrated parameters.
\end{lemma}

\begin{proof}
Let $\mu_1 =\frac{1}{\mathcal{V}^{1/2}+\epsilon}$, $\mu_2 =  \frac{1}{\epsilon}$, $\mu_3 = \frac{1}{(\mathcal{V}+\epsilon)^p}$, $\mu_4 = \frac{1}{\epsilon^p}$,
$\mu_5 =\frac{\beta}{\log (1+exp(\beta\sqrt{\mathcal{V}}))}$, $\mu_6 = \frac{\beta}{\log 2}$, then we can get the above result.
\end{proof}
From the definition of $\mathcal{V}$, and regard $p_i, \sigma_i, G_i$ as constants, we then have,

\begin{center}
\begin{tabular}{ |c||c|c|c| } 
 \hline
  Bounds & $\epsilon$-FedAdam & $p$-FedAdam & $s$-FedAdam \\ 
 \hline\hline
 $\mu_{lower}$ & $ O(\frac{1}{K \gamma_t \sqrt{1+\frac{1}{S}}})$ & $ O(\frac{1}{(K\gamma_t)^{2p} (1+\frac{1}{S})^{p}})$ & $ O(\frac{1}{K \gamma_t \sqrt{1+\frac{1}{S}}})$ \\ 
 \hline
 $\mu_{upper}$ & $O(\frac{1}{\epsilon})$ & $O(\frac{1}{\epsilon^p})$ &  $ O(\beta)$ \\ 
 \hline
\end{tabular}
\end{center}

\begin{remark}
Adaptive learning rate pairs $(\mu_{lower}, \mu_{upper})$ are related with algorithm's calibrated parameters (i.e., $\epsilon$, $p$, $\beta$), inner loop iterations $K$, { inner loop stepsize $\gamma_t$} and participated device numbers $S$.
\end{remark}

\subsection{$\epsilon$-FedAdam Convergence Analysis with Paritical Device Participation in Nonconvex Setting}

This time we build a complicated auxiliary sequence for FedAdam.

\begin{lemma}\label{z_t}
Define $z_{t} = x_{t} + \frac{\beta_1}{1-\beta_1}(x_{t} - x_{t-1}), \forall t \geq 1$ $\beta_1 \in [0,1)$. Then the following updating formula holds for $\epsilon$-FedAdam optimizer: 
\begin{equation}
    z_{t+1}=z_{t} + \frac{\eta \beta_1}{1-\beta_1}
    (\frac{1}{v_{t-1}^{1/2}+\epsilon}-\frac{1}{v_{t}^{1/2}+\epsilon})\odot m_{t-1}-\frac{\eta}{v_{t}^{1/2}+\epsilon}\odot \Delta_t ;
\end{equation}
\end{lemma}

\begin{proof}
\begin{align*}
    z_{t+1} &=x_{t+1} + \frac{\beta_1}{1-\beta_1}(x_{t+1} - x_{t})\\
    z_{t+1} &= z_{t} +\frac{1}{1-\beta_1}(x_{t+1} - x_{t})-\frac{\beta_1}{1-\beta_1}(x_{t} - x_{t-1})\\
    &=z_{t} -\frac{1}{1-\beta_1}\frac{\eta}{v_{t}^{1/2}+\epsilon} \odot m_t +\frac{\beta_1}{1-\beta_1}\frac{\eta}{v_{t-1}^{1/2}+\epsilon}\odot  m_{t-1}\\
    &=z_{t} + \frac{\eta \beta_1}{1-\beta_1}
    (\frac{1}{v_{t-1}^{1/2}+\epsilon}-\frac{1}{v_{t}^{1/2}+\epsilon})\odot m_{t-1}-\frac{\eta}{v_{t}^{1/2}+\epsilon}\odot \Delta_t\\
\end{align*}
\end{proof}

\begin{lemma}\label{square}
As defined in Lemma \ref{z_t}, with the condition that $v_t\geq v_{t-1}$, we can derive the bound of distance of  $E[\|z_{t+1}-z_t\|^2]$ as follows:
\begin{align}\label{eq:distance}
     E[\|z_{t+1}-z_t\|^2]&\leq \frac{2\eta^2\beta_1^2 \mathcal{V}}{(1-\beta_1)^2}E[\sum_{j=1}^d (\frac{1}{v_{t-1,j}^{1/2}+\epsilon})^2-(\frac{1}{v_{t,j}^{1/2}+\epsilon})^2]
    +2\eta^2 \mu_2^2 \mathcal{V}.
\end{align}
\end{lemma}

\begin{proof}

\begin{align*}
   E[\|z_{t+1}-z_t\|^2]&= E[\|\frac{\eta \beta_1}{1-\beta_1}
    (\frac{1}{v_{t-1}^{1/2}+\epsilon}-\frac{1}{ v_{t}^{1/2}+\epsilon})\odot m_{t-1}-\frac{\eta}{v_t^{1/2}+\epsilon}\odot \Delta_t\|^2]\\
    &\leq 2E[\|\frac{\eta \beta_1}{1-\beta_1}
    (\frac{1}{v_{t-1}^{1/2}+\epsilon}-\frac{1}{v_{t}^{1/2}+\epsilon})\odot m_{t-1}\|]^2 +2E[\|\frac{\eta}{v_t^{1/2}+\epsilon}\odot \Delta_t\|]^2\\
    &\leq \frac{2\eta^2\beta_1^2 \mathcal{V}}{(1-\beta_1)^2}E[\sum_{j=1}^d (\frac{1}{v_{t-1,j}^{1/2}+\epsilon}-\frac{1}{v_{t,j}^{1/2}+\epsilon})^2]
    + 2\eta^2 \mu_2^2\mathcal{V} \\
    &\leq \frac{2\eta^2\beta_1^2 \mathcal{V}}{(1-\beta_1)^2}E[\sum_{j=1}^d (\frac{1}{v_{t-1,j}^{1/2}+\epsilon})^2-(\frac{1}{v_{t,j}^{1/2}+\epsilon})^2]
    +2\eta^2 \mu_2^2 \mathcal{V}\\
\end{align*}
The first inequality holds because $\|a-b\|^2\leq 2\|a\|^2 +2 \|b\|^2$, the second inequality holds because Lemma \ref{delta_t}, Lemma \ref{m_t} and Lemma \ref{bound1},  the third inequality holds because $(a-b)^2 \leq a^2- b^2$ when $a\geq b$, and in our assumption, we have $v_t \geq v_{t-1}$ holds.

\end{proof}

\begin{lemma}\label{product}
As defined in Lemma \ref{z_t}, with the condition that $v_t\geq v_{t-1}$, we can derive the bound of the inner product as follows:
\begin{align} \label{equ:product3}
      -E[\langle\nabla f(z_t)-\nabla f(x_t), \frac{\eta}{ v_t^{1/2}+\epsilon}\odot \Delta_t\rangle]
    \leq \frac{1}{2}L^2\eta^2\mu_2^2 (\frac{\beta_1}{1-\beta_1})^2 \mathcal{V} +  \frac{1}{2}\eta^2\mu_2^2\mathcal{V}.
\end{align}
\end{lemma}

\begin{proof}

\begin{align*}
    -E[\langle\nabla f(z_t)&-\nabla f(x_t), \frac{\eta}{ v_t^{1/2}+\epsilon}\odot \Delta_t\rangle]\\
    &\leq \frac{1}{2}  E[\|\nabla f(z_t)-\nabla f(x_t)\|^2]+ \frac{1}{2}E[\|\frac{\eta}{ v_t^{1/2}+\epsilon}\odot \Delta_t\|^2]\\
    &\leq \frac{L^2}{2}E[\|z_t- x_t\|^2]+
    \frac{1}{2}E[\|\frac{\eta}{ v_t^{1/2}+\epsilon}\odot \Delta_t\|^2]\\
    &=\frac{L^2}{2}(\frac{\beta_1}{1-\beta_1})^2E[\|x_t- x_{t-1}\|^2]+ 
   \frac{1}{2}E[\|\frac{\eta}{ v_t^{1/2}+\epsilon}\odot \Delta_t\|^2]\\
    &=\frac{L^2 }{2}(\frac{\beta_1}{1-\beta_1})^2E[\|\frac{\eta}{ v_{t-1}^{1/2}+\epsilon}\odot m_{t-1}\|^2]+ \frac{1}{2}E[\|\frac{\eta}{ v_t^{1/2}+\epsilon}\odot \Delta_t\|^2]\\
   &\leq \frac{1}{2}L^2\eta^2\mu_2^2 (\frac{\beta_1}{1-\beta_1})^2 \mathcal{V} +  \frac{1}{2}\eta^2\mu_2^2\mathcal{V}
\end{align*}

The first inequality holds because $\frac{1}{2}a^2+\frac{1}{2}b^2\geq -<a,b>$, the second inequality holds for L-smoothness, the last inequalities hold due to Lemma \ref{delta_t}, \ref{m_t} and \ref{bound1}.
\end{proof}

\begin{lemma} \label{lemma:drift}
\begin{align}
    E[\langle \nabla& f(x_t), \frac{\eta}{ v_t^{1/2}+\epsilon} \odot(K \gamma_t \nabla f(x_t) - \Delta_t)\rangle] \nonumber\\
     &\leq \frac{\eta K \gamma_t}{2} \|\frac{\nabla f(x_t)}{\sqrt{v_t^{1/2}+\epsilon}}\|^2+  \frac{L \eta\mu_2}{2}  (5K^2 \gamma_t^3 (\sum_{i=1}^N p_i \sigma_i^2 + 2K\sigma_g^2) + 10K^3 \gamma_t^3 E[\|\nabla f( x_t)\|^2]).
\end{align}
\end{lemma}

\begin{proof}
\begin{align*}
    E[\langle \nabla& f(x_t), \frac{\eta}{ v_t^{1/2}+\epsilon} \odot(K \gamma_t \nabla f(x_t) - \Delta_t)\rangle]\\
    &= \langle \nabla f(x_t), \frac{\eta}{ v_t^{1/2}+\epsilon} \odot(K \gamma_t \nabla f(x_t) - \sum_{i=1}^{N}p_i \sum_{k=0}^{K-1} \gamma_t E[g_{t,k}^{(i)}])\rangle\\
    &\leq \frac{\eta K \gamma_t }{2} \|\frac{\nabla f(x_t)}{\sqrt{v_t^{1/2}+\epsilon}}\|^2+  \frac{\eta \gamma_t}{2K} \|\frac{K\nabla f(x_t) - \sum_{i=1}^{N}p_i \sum_{k=0}^{K-1} \nabla f_i(x_{t,k}^{(i)})}{\sqrt{v_t^{1/2}+\epsilon}}\|^2\\
    &\leq \frac{\eta K \gamma_t}{2} \|\frac{\nabla f(x_t)}{\sqrt{v_t^{1/2}+\epsilon}}\|^2+  \frac{\eta \gamma_t \mu_2}{2K} \|K\nabla f(x_t) -\sum_{i=1}^{N}p_i \sum_{k=0}^{K-1} \nabla f_i(x_{t,k}^{(i)})\|^2\\
    &\leq \frac{\eta K \gamma_t}{2} \|\frac{\nabla f(x_t)}{\sqrt{v_t^{1/2}+\epsilon}}\|^2+  \frac{\eta \gamma_t \mu_2}{2} \sum_{i=1}^{N}p_i \sum_{k=0}^{K-1}\|\nabla f(x_t) - \nabla f_i(x_{t,k}^{(i)})\|^2\\
     &\leq \frac{\eta K \gamma_t}{2} \|\frac{\nabla f(x_t)}{\sqrt{v_t^{1/2}+\epsilon}}\|^2+  \frac{L\eta \gamma_t \mu_2}{2} \sum_{i=1}^{N}p_i \sum_{k=0}^{K-1}\|x_t - x_{t,k}^{(i)}\|^2.
\end{align*}

Following Lemma 3 in \cite{reddi2020adaptive}, we can derive the corresponding drift bound for our problem that for $\gamma_t \leq \frac{1}{8LK}$
\begin{align}\label{eq:drift}
    \sum_{i=1}^{N}p_i \|x_t - x_{t,k}^{(i)}\|^2 \leq 5K \gamma_t^2 (\sum_{i=1}^{N} p_i \sigma_i^2 + 2K\sigma_g^2) + 10K^2\gamma_t^2 E[\|\nabla f( x_t)\|^2].
\end{align}

Then we get the upper bound,
\begin{align*}
    E[\langle \nabla& f(x_t), \frac{\eta}{ v_t^{1/2}+\epsilon} \odot(K \gamma_t\nabla f(x_t) - \Delta_t)\rangle]\\
     &\leq \frac{\eta K \gamma_t}{2} \|\frac{\nabla f(x_t)}{\sqrt{v_t^{1/2}+\epsilon}}\|^2+  \frac{L\eta\mu_2}{2}  (5K^2 \gamma_t^3  (\sum_{i=1}^{N} p_i \sigma_i^2 + 2K\sigma_g^2) + 10K^3\gamma_t^3 E[\|\nabla f( x_t)\|^2]).
\end{align*}

\end{proof}

\noindent\textbf{Proof of $\epsilon$-FedAdam with partial device participation in nonconvex setting.}


\begin{proof}
 From L-smoothness and Lemma \ref{z_t}, we have

\begin{align*}
    f(z_{t+1})&\leq f(z_t) + \langle\nabla f(z_t), z_{t+1}-z_t\rangle+\frac{L}{2}\|z_{t+1}-z_t\|^2\\
    &= f(z_t)+ \frac{\eta \beta_1}{1-\beta_1}
    \langle\nabla f(z_t), (\frac{1}{ v_{t-1}^{1/2}+\epsilon}-\frac{1}{ v_{t}^{1/2}+\epsilon}) \odot m_{t-1}\rangle\\
    &- \langle\nabla f(z_t), \frac{\eta}{ v_t^{1/2}+\epsilon} \odot \Delta_t\rangle +\frac{L}{2}\|z_{t+1}-z_t\|^2\\
\end{align*}

Take expectation on both sides,

\begin{align*}
    E[f(z_{t+1})-f(z_t)] &\leq  \frac{\eta \beta_1}{1-\beta_1}
    E[\langle\nabla f(z_t), (\frac{1}{ v_{t-1}^{1/2}+\epsilon}-\frac{1}{ v_{t}^{1/2}+\epsilon}) \odot m_{t-1}\rangle]\\
     &- E[\langle\nabla f(z_t), \frac{\eta}{ v_t^{1/2}+\epsilon} \odot \Delta_t\rangle] +\frac{L}{2}E[\|z_{t+1}-z_t\|^2]\\ 
     &= \frac{\eta \beta_1}{1-\beta_1}
    E[\langle\nabla f(z_t), (\frac{1}{ v_{t-1}^{1/2}+\epsilon}-\frac{1}{ v_{t}^{1/2}+\epsilon}) \odot m_{t-1}\rangle]\\
     &- E[\langle\nabla f(z_t)-\nabla f(x_t), \frac{\eta}{ v_t^{1/2}+\epsilon} \odot \Delta_t\rangle] - E[\langle\nabla f(x_t), \frac{\eta}{ v_t^{1/2}+\epsilon} \odot \Delta_t\rangle]\\
     &+\frac{L}{2}E[\|z_{t+1}-z_t\|^2].
\end{align*}

Plug in the results from prepared lemmas, then we have,
\begin{align*}
    E[f(z_{t+1})-f(z_t)] &\leq   
     \frac{\eta \beta_1}{1-\beta_1}
    E[\langle\nabla f(z_t), (\frac{1}{ v_{t-1}^{1/2}+\epsilon}-\frac{1}{ v_{t}^{1/2}+\epsilon}) \odot m_{t-1}\rangle]\\
     &+\frac{1}{2}L^2\eta^2\mu_2^2 (\frac{\beta_1}{1-\beta_1})^2 \mathcal{V} +  +\frac{1}{2}\eta^2\mu_2^2\mathcal{V} - E[\langle\nabla f(x_t), \frac{\eta}{ v_t^{1/2}+\epsilon} \odot \Delta_t\rangle]\\
     &+  \frac{L\eta^2\beta_1^2 \mathcal{V}}{(1-\beta_1)^2}E[\sum_{j=1}^d (\frac{1}{v_{t-1,j}^{1/2}+\epsilon})^2-(\frac{1}{v_{t,j}^{1/2}+\epsilon})^2] + L\eta^2 \mu_2^2 \mathcal{V}.
\end{align*}

Applying the bound of $E[\|\nabla f(z_{t})\|^2] \leq \sum_{i=1}^N p_i G_i^2 < \mathcal{V}$ and $E[\|m_{t-1}\|^2] \leq  \mathcal{V}$ for $\forall t\in\{0,...,T-1\}$, we have  
\begin{align*}
    E[\langle\nabla f(z_{t}),  (\frac{1}{v_{t-1}^{1/2}+\epsilon}-\frac{1}{v_{t}^{1/2}+\epsilon}) \odot m_{t-1} \rangle] 
    &\leq E[\|\nabla f(z_{t})\|\| m_{t-1}\|][\sum_{j=1}^d \frac{1}{v_{t-1,j}^{1/2}+\epsilon}-\frac{1}{v_{t,j}^{1/2}+\epsilon}]\\
    &\leq \sqrt{\sum_{i=1}^N p_i G_i^2 \cdot  \mathcal{V}} [\sum_{j=1}^d \frac{1}{v_{t-1,j}^{1/2}+\epsilon}-\frac{1}{v_{t,j}^{1/2}+\epsilon}]\\
    &\leq   \mathcal{V} [\sum_{j=1}^d \frac{1}{v_{t-1,j}^{1/2}+\epsilon}-\frac{1}{v_{t,j}^{1/2}+\epsilon}]
\end{align*}

Then we derive,
\begin{align*}
    E[f(z_{t+1})-f(z_t)] 
    &\leq \frac{\eta \beta_1}{1-\beta_1}
     \mathcal{V} [\sum_{j=1}^d \frac{1}{v_{t-1,j}^{1/2}+\epsilon}-\frac{1}{v_{t,j}^{1/2}+\epsilon}]+\frac{1}{2}L^2\eta^2\mu_2^2 (\frac{\beta_1}{1-\beta_1})^2 \mathcal{V} \\ &+\frac{1}{2}\eta^2\mu_2^2\mathcal{V} - E[\langle\nabla f(x_t), \frac{\eta}{ v_t^{1/2}+\epsilon} \odot \Delta_t\rangle]\\
     &+ \frac{L\eta^2\beta_1^2 \mathcal{V}}{(1-\beta_1)^2}E[\sum_{j=1}^d (\frac{1}{v_{t-1,j}^{1/2}+\epsilon})^2-(\frac{1}{v_{t,j}^{1/2}+\epsilon})^2] + L\eta^2 \mu_2^2 \mathcal{V}
\end{align*}

$E[\langle\nabla f(x_t), \frac{\eta}{ v_t^{1/2}+\epsilon} \odot \Delta_t\rangle] = E[\langle\nabla f(x_t), \frac{\eta K \gamma_t}{ v_t^{1/2}+\epsilon} \odot \nabla f(x_t) \rangle - \langle\nabla f(x_t), \frac{\eta}{ v_t^{1/2}+\epsilon} \odot (K \gamma_t \nabla f(x_t) - \Delta_t)\rangle]$

\begin{align*}
    E[f(z_{t+1})-f(z_t)]  
    &\leq\frac{\eta \beta_1}{1-\beta_1}
     \mathcal{V} [\sum_{j=1}^d \frac{1}{v_{t-1,j}^{1/2}+\epsilon}-\frac{1}{v_{t,j}^{1/2}+\epsilon}] +\frac{1}{2}L^2\eta^2\mu_2^2 (\frac{\beta_1}{1-\beta_1})^2 \mathcal{V} \\ &+\frac{1}{2}\eta^2\mu_2^2\mathcal{V} - E[\langle\nabla f(x_t), \frac{\eta K \gamma_t }{ v_t^{1/2}+\epsilon} \odot \nabla f(x_t)\rangle] \\
     &+ \frac{\eta K \gamma_t }{2} \|\frac{\nabla f(x_t)}{\sqrt{v_t^{1/2}+\epsilon}}\|^2+  \frac{\eta\mu_2}{2}  (5K^2 \gamma_t^3 (\sum_{i=1}^N p_i \sigma_i^2 + 2K\sigma_g^2) + 10K^3\gamma_t^3 E[\|\nabla f( x_t)\|^2])\\
    &+ \frac{L\eta^2\beta_1^2 \mathcal{V}}{(1-\beta_1)^2}E[\sum_{j=1}^d (\frac{1}{v_{t-1,j}^{1/2}+\epsilon})^2-(\frac{1}{v_{t,j}^{1/2}+\epsilon})^2] + L\eta^2 \mu_2^2\mathcal{V}\\
      &\leq\frac{\eta \beta_1}{1-\beta_1}
     \mathcal{V} [\sum_{j=1}^d \frac{1}{v_{t-1,j}^{1/2}+\epsilon}-\frac{1}{v_{t,j}^{1/2}+\epsilon}] +\frac{1}{2}L^2\eta^2\mu_2^2 (\frac{\beta_1}{1-\beta_1})^2 \mathcal{V} +\frac{1}{2}\eta^2\mu_2^2\mathcal{V} \\
     &- \frac{\eta K \gamma_t \mu_1}{2} \|\nabla f(x_t)\|^2+  \frac{\eta\mu_2}{2}  (5K^2 \gamma_t^3 (\sum_{i=1}^N p_i \sigma_i^2 + 2K\sigma_g^2) + 10K^3\gamma_t^3 E[\|\nabla f( x_t)\|^2])\\
     &+ \frac{L\eta^2\beta_1^2 \mathcal{V}}{(1-\beta_1)^2}E[\sum_{j=1}^d (\frac{1}{v_{t-1,j}^{1/2}+\epsilon})^2-(\frac{1}{v_{t,j}^{1/2}+\epsilon})^2] + L\eta^2 \mu_2^2\mathcal{V}\\
   &\leq\frac{\eta \beta_1}{1-\beta_1}
     \mathcal{V} [\sum_{j=1}^d \frac{1}{v_{t-1,j}^{1/2}+\epsilon}-\frac{1}{v_{t,j}^{1/2}+\epsilon}] +\frac{1}{2}L^2\eta^2\mu_2^2 (\frac{\beta_1}{1-\beta_1})^2 \mathcal{V} +\frac{3}{2}\eta^2\mu_2^2\mathcal{V} \\
     &- (\frac{\eta K \gamma_t \mu_1}{2} - 5\eta \mu_2 K^3 \gamma_t^3 )\|\nabla f(x_t)\|^2+  \frac{5\eta\mu_2 K^2 \gamma_t^3}{2}   (\sum_{i=1}^N p_i \sigma_i^2 + 2K\sigma_g^2)\\
     &+ \frac{L\eta^2\beta_1^2 \mathcal{V}}{(1-\beta_1)^2}E[\sum_{j=1}^d (\frac{1}{v_{t-1,j}^{1/2}+\epsilon})^2-(\frac{1}{v_{t,j}^{1/2}+\epsilon})^2] 
\end{align*}
Then, we have, 
\begin{align*}
   & (\frac{\eta K \gamma_t \mu_1}{2} - 5\eta \mu_2 K^3 \gamma_t^3 )\|\nabla f(x_t)\|^2 \\
    & \leq E[f(z_t) -f(z_{t+1})] +\frac{\eta \beta_1}{1-\beta_1}
     \mathcal{V} [\sum_{j=1}^d \frac{1}{v_{t-1,j}^{1/2}+\epsilon}-\frac{1}{v_{t,j}^{1/2}+\epsilon}] +\frac{1}{2}L^2\eta^2\mu_2^2 (\frac{\beta_1}{1-\beta_1})^2 \mathcal{V} +\frac{3}{2}\eta^2\mu_2^2\mathcal{V} \\
     &+   \frac{5\eta\mu_2 K^2 \gamma_t^3}{2}   (\sum_{i=1}^N p_i \sigma_i^2 + 2K\sigma_g^2)
     + \frac{L\eta^2\beta_1^2 \mathcal{V}}{(1-\beta_1)^2}E[\sum_{j=1}^d (\frac{1}{v_{t-1,j}^{1/2}+\epsilon})^2-(\frac{1}{v_{t,j}^{1/2}+\epsilon})^2] 
\end{align*}

Require $\frac{\eta K \gamma_t \mu_1}{2} - 5\eta \mu_2 K^3 \gamma_t^3 > 0$, then $\gamma_t < \sqrt{\frac{\mu_1}{10\mu_2 K^2}} $. We further derive the bound of inner loop stepsize as $\gamma_t < \min\{ \frac{1}{8LK},\sqrt {\frac{\mu_1}{10\mu_2 K^2}} \}$, or we can get $K \gamma_t < \min\{\frac{1}{8L}, \sqrt{\frac{\mu_1}{10\mu_2}}\}$.

Let $\frac{\frac{5\eta\mu_2 K^2 \gamma_t^3}{2} }{ \frac{\eta K \gamma_t \mu_1}{2} - 5\eta \mu_2 K^3 \gamma_t^3 } = \frac{1}{T}$, then we have $\gamma_t = \sqrt{\frac{\mu_1}{5T\mu_2 K +10 \mu_2 K^2}} = \Theta (\sqrt{\frac{\mu_1}{\mu_2 T K}})$, 
$$\frac{1}{\frac{\eta K \gamma_t \mu_1}{2} - 5\eta \mu_2 K^3 \gamma_t^3}  = O(\frac{1}{T\eta \mu_2 K^2 \gamma_t^3}) = O(\sqrt{\frac{\mu_2 T}{\mu_1^3 \eta K}})$$

\begin{align*}
      \|\nabla f(x_t)\|^2 
    & \leq O(\sqrt{\frac{\mu_2 T}{\mu_1^3 \eta K}}) E[f(z_t) -f(z_{t+1})] +  O(\sqrt{\frac{\mu_2 T}{\mu_1^3 \eta K}}) \frac{\eta \beta_1}{1-\beta_1}
     \mathcal{V} [\sum_{j=1}^d \frac{1}{v_{t-1,j}^{1/2}+\epsilon}-\frac{1}{v_{t,j}^{1/2}+\epsilon}]\\ &+O(\sqrt{\frac{\mu_2 T}{\mu_1^3 \eta K}}) (\frac{1}{2}L^2\eta^2\mu_2^2 (\frac{\beta_1}{1-\beta_1})^2 \mathcal{V} +\frac{3}{2}\eta^2\mu_2^2\mathcal{V} )\\
     &+ O(\frac{1}{T})(\sum_{i=1}^N p_i \sigma_i^2 + 2K\sigma_g^2)
     +  O(\sqrt{\frac{\mu_2 T}{\mu_1^3 \eta K}})\frac{L\eta^2\beta_1^2 \mathcal{V}}{(1-\beta_1)^2}E[\sum_{j=1}^d (\frac{1}{v_{t-1,j}^{1/2}+\epsilon})^2-(\frac{1}{v_{t,j}^{1/2}+\epsilon})^2] 
\end{align*}

Sum from $t =0$ to $T-1$ and multiply by $\frac{1}{T}$,  because $z_0 = x_0$ and Lemma \ref{bound1},
\begin{align*}
    \frac{1}{T}\sum_{t=0}^{T-1} \|\nabla f(x_t)\|^2 
    & \leq O(\sqrt{\frac{\mu_2 }{\mu_1^3 \eta K T}}) [f(x_0) -f^*] + O(\sqrt{\frac{\mu_2 }{\mu_1^3 \eta K T}})\frac{\eta \beta_1}{1-\beta_1}
     \mathcal{V} \sum_{j=1}^d (\mu_2 - \mu_1)\\
     &+O(\sqrt{\frac{\mu_2 T}{\mu_1^3 \eta K}})(\frac{1}{2}L^2\eta^2\mu_2^2 (\frac{\beta_1}{1-\beta_1})^2 \mathcal{V} +\frac{3}{2}\eta^2\mu_2^2\mathcal{V} )\\
     &+ O(\frac{1}{T})(\sum_{i=1}^N p_i \sigma_i^2 + 2K\sigma_g^2)
     + O(\sqrt{\frac{\mu_2  }{\mu_1^3 \eta K T}})\frac{L\eta^2\beta_1^2 \mathcal{V}}{(1-\beta_1)^2}\sum_{j=1}^d (\mu_2^2 -\mu_1^2) \\
     &= O([f(x_0) -f^*]\sqrt{\frac{\mu_2 }{\mu_1^3 \eta K T}} +  \mathcal{V} \sqrt{\frac{\mu_2  \eta}{\mu_1^3  K T}}
      \sum_{j=1}^d (\mu_2 - \mu_1)
     + \sqrt{\frac{\mu_2^5 \eta^3 T}{\mu_1^3 K}} \mathcal{V} \\
     &+ \frac{1}{T}(\sum_{i=1}^N p_i \sigma_i^2 + 2K\sigma_g^2)
     +  \sqrt{\frac{\mu_2  \eta^3  }{\mu_1^3 K T}}\mathcal{V} \sum_{j=1}^d (\mu_2^2 -\mu_1^2))
\end{align*}

Notice that  $\mathcal{V} = O(K^2  \gamma_t^2 (1+\frac{1}{S})) = O ((1+\frac{1}{S}) \frac{\mu_1 K}{\mu_2 T})$, neglect constant $\sigma_i$, 

\begin{align*}
    \frac{1}{T}\sum_{t=0}^{T-1}  \|\nabla f(x_{t})\|^2 &\leq  O([f(x_0) -f^*]\sqrt{\frac{\mu_2 }{\mu_1^3 \eta K T}} +   (1+\frac{1}{S})\sqrt{\frac{\eta K}{\mu_1 \mu_2 T^3}}
      \sum_{j=1}^d (\mu_2 - \mu_1)
     + (1+\frac{1}{S})\sqrt{\frac{\mu_2^3 \eta^3 K}{\mu_1 T}} \\
     &+ \frac{K\sigma_g^2}{T}
     +  (1+\frac{1}{S})\sqrt{\frac{ K \eta^3  }{\mu_1 \mu_2 T^3}} \sum_{j=1}^d (\mu_2^2 -\mu_1^2))\\
     &\leq  O( \sqrt{\frac{\mu_2 }{\mu_1^3 \eta K T}} + \frac{K\sigma_g^2}{T} 
     +  (1+\frac{1}{S})(\sqrt{\frac{\mu_2 \eta K}{\mu_1  T^3}}
     + \sqrt{\frac{\mu_2^3 \eta^3 K}{\mu_1 T}}  
     +  \sqrt{\frac{\mu_2^3 K \eta^3  }{\mu_1 T^3}}))\\
     &\leq  O( \sqrt{\frac{\mu_2 }{\mu_1^3 \eta K T}} + \frac{K\sigma_g^2}{T} 
     +  (1+\frac{1}{S})\sqrt{\frac{\mu_2^3 \eta^3 K}{\mu_1 T}}) .
\end{align*}

Our $\epsilon$-FedAdam methods are proved to converge with convergence rate of $O(\frac{1}{\sqrt{T}})$. And we get the result that, if $K \gamma_t <  O(\min\{ \frac{1}{L },\sqrt{\frac{\mu_1}{\mu_2}} \})$, Federated AGMs always converge.

Consider the calibration parameter $(\mu_{lower}, \mu_{upper})$ in FedAdam, and we further require $\eta = \frac{1}{K}$;

for vinilla FedAdam and $\epsilon$-FedAdam, $K\gamma_t < O( \min\{ \frac{1}{L},\sqrt[3]{\frac{\epsilon}{   \sqrt{1+\frac{1}{S}}}} \})$,
$$ \min_{t=0,...,T-1} \|\nabla f(x_{t})\|^2 \leq O(\sqrt{\frac{(1+1/S)^{3/2}}{\epsilon T}} + \frac{K\sigma_g^2}{T} + \sqrt{\frac{(1+1/S)^{5/2}}{\epsilon^3 K^2 T}} ).$$
\end{proof}

Thus, we get the sublinear convergence rate of FedAdam, $\epsilon$-FedAdam in nonconvex setting with full device participation,  which can recover gradient-based FedAvg methods (let $\beta_1=0, \beta_2=0$) and momentum-based FedMomentum methods  (let $\beta_2=0 $). 

\subsection{$p$-FedAdam Convergence Analysis with Paritical Device Participation in Nonconvex Setting}

Similarly, we build a complicated auxiliary sequence for $p$-FedAdam.

\begin{lemma}\label{z_t_1}
Define $z_{t} = x_{t} + \frac{\beta_1}{1-\beta_1}(x_{t} - x_{t-1}), \forall t \geq 1$ $\beta_1 \in [0,1)$. Then the following updating formula holds for $p$-FedAdam optimizer: 
\begin{equation}
    z_{t+1}=z_{t} + \frac{\eta \beta_1}{1-\beta_1}
    (\frac{1}{(v_{t-1}+\epsilon)^{p}}-\frac{1}{(v_{t}+\epsilon)^{p}})\odot m_{t-1}-\frac{\eta}{(v_{t}+\epsilon)^{p}}\odot \Delta_t ;
\end{equation}
\end{lemma}

\begin{proof}
Similar as Lemma \ref{z_t}, we have
\begin{align*}
    z_{t+1} &=x_{t+1} + \frac{\beta_1}{1-\beta_1}(x_{t+1} - x_{t})\\
    z_{t+1} &= z_{t} +\frac{1}{1-\beta_1}(x_{t+1} - x_{t})-\frac{\beta_1}{1-\beta_1}(x_{t} - x_{t-1})\\
    &=z_{t} -\frac{1}{1-\beta_1}\frac{\eta}{(v_{t}+\epsilon)^{p}} \odot m_t +\frac{\beta_1}{1-\beta_1}\frac{\eta}{(v_{t-1}+\epsilon)^p}\odot  m_{t-1}\\
    &=z_{t} + \frac{\eta \beta_1}{1-\beta_1}
    (\frac{1}{(v_{t-1}+\epsilon)^p}-\frac{1}{(v_{t}+\epsilon)^{p}})\odot m_{t-1}-\frac{\eta}{(v_{t}+\epsilon)^{p}}\odot \Delta_t\\
\end{align*}
\end{proof}

After building the auxiliary sequence for $p$-FedAdam, we can similarly have the following lemmas, whose proof can be easily derived from the above analysis of $\epsilon$-FedAdam.

\begin{lemma}\label{square_1}
As defined in Lemma \ref{z_t_1}, with the condition that $v_t\geq v_{t-1}$, we can derive the bound of distance of  $E[\|z_{t+1}-z_t\|^2]$ as follows:
\begin{align}\label{eq:distance}
     E[\|z_{t+1}-z_t\|^2]&\leq \frac{2\eta^2\beta_1^2 \mathcal{V}}{(1-\beta_1)^2}E[\sum_{j=1}^d (\frac{1}{(v_{t-1,j}+\epsilon)^p})^2-(\frac{1}{(v_{t,j}+\epsilon)^p})^2]
    +2\eta^2 \mu_4^2 \mathcal{V}.
\end{align}
\end{lemma}

\begin{lemma}\label{product_1}
As defined in Lemma \ref{z_t_1}, with the condition that $v_t\geq v_{t-1}$, we can derive the bound of the inner product as follows:
\begin{align} \label{equ:product3}
      -E[\langle\nabla f(z_t)-\nabla f(x_t), \frac{\eta}{ (v_t+\epsilon)^p }\odot \Delta_t\rangle]
    \leq \frac{1}{2}L^2\eta^2\mu_4^2 (\frac{\beta_1}{1-\beta_1})^2 \mathcal{V} +  \frac{1}{2}\eta^2\mu_4^2\mathcal{V}.
\end{align}
\end{lemma}

\begin{lemma} \label{lemma:drift_1}
\begin{align}
    E[\langle \nabla& f(x_t), \frac{\eta}{ (v_{t}+\epsilon)^p} \odot(K \gamma_t \nabla f(x_t) - \Delta_t)\rangle] \nonumber\\
     &\leq \frac{\eta K \gamma_t}{2} \|\frac{\nabla f(x_t)}{\sqrt{(v_t+\epsilon)^p}}\|^2+  \frac{L \eta\mu_4}{2}  (5K^2 \gamma_t^3 (\sum_{i=1}^N p_i \sigma_i^2 + 2K\sigma_g^2) + 10K^3 \gamma_t^3 E[\|\nabla f( x_t)\|^2]).
\end{align}
\end{lemma}

\noindent\textbf{Proof of $p$-FedAdam with partial device participation in nonconvex setting.}

We then study the convergence analysis of $p$-FedAdam methods, which also include $\epsilon$-FedAdam methods.

\begin{proof}
 From L-smoothness and the sequence derived in Lemma \ref{z_t_1}, we have

\begin{align*}
    f(z_{t+1})&\leq f(z_t) + \langle\nabla f(z_t), z_{t+1}-z_t\rangle+\frac{L}{2}\|z_{t+1}-z_t\|^2\\
    &= f(z_t)+ \frac{\eta \beta_1}{1-\beta_1}
    \langle\nabla f(z_t), (\frac{1}{   (v_{t-1}+\epsilon)^p}-\frac{1}{   (v_{t}+\epsilon)^p}) \odot m_{t-1}\rangle\\
    &- \langle\nabla f(z_t), \frac{\eta}{ (v_{t}+\epsilon)^p} \odot \Delta_t\rangle +\frac{L}{2}\|z_{t+1}-z_t\|^2\\
\end{align*}

Take expectation on both sides,

\begin{align*}
    E[f(z_{t+1})-f(z_t)] &\leq  \frac{\eta \beta_1}{1-\beta_1}
    E[\langle\nabla f(z_t), (\frac{1}{   (v_{t-1}+\epsilon)^p}-\frac{1}{   (v_{t}+\epsilon)^p}) \odot m_{t-1}\rangle]\\
     &- E[\langle\nabla f(z_t), \frac{\eta}{ (v_{t}+\epsilon)^p} \odot \Delta_t\rangle] +\frac{L}{2}E[\|z_{t+1}-z_t\|^2]\\ 
     &= \frac{\eta \beta_1}{1-\beta_1}
    E[\langle\nabla f(z_t), (\frac{1}{   (v_{t-1}+\epsilon)^p}-\frac{1}{   (v_{t}+\epsilon)^p}) \odot m_{t-1}\rangle]\\
     &- E[\langle\nabla f(z_t)-\nabla f(x_t), \frac{\eta}{ (v_{t}+\epsilon)^p} \odot \Delta_t\rangle] - E[\langle\nabla f(x_t), \frac{\eta}{ (v_{t}+\epsilon)^p} \odot \Delta_t\rangle]\\
     &+\frac{L}{2}E[\|z_{t+1}-z_t\|^2].
\end{align*}

Plug in the results from prepared lemmas, then we have,
\begin{align*}
    E[f(z_{t+1})-f(z_t)] &\leq   
     \frac{\eta \beta_1}{1-\beta_1}
    E[\langle\nabla f(z_t), (\frac{1}{   (v_{t-1}+\epsilon)^p}-\frac{1}{   (v_{t}+\epsilon)^p}) \odot m_{t-1}\rangle]\\
     &+\frac{1}{2}L^2\eta^2\mu_4^2 (\frac{\beta_1}{1-\beta_1})^2 \mathcal{V} +  +\frac{1}{2}\eta^2\mu_4^2\mathcal{V} - E[\langle\nabla f(x_t), \frac{\eta}{ (v_{t}+\epsilon)^p} \odot \Delta_t\rangle]\\
     &+  \frac{L\eta^2\beta_1^2 \mathcal{V}}{(1-\beta_1)^2}E[\sum_{j=1}^d (\frac{1}{(v_{t-1,j}+\epsilon)^p})^2-(\frac{1}{(v_{t,j}+\epsilon)^p})^2] + L\eta^2 \mu_4^2 \mathcal{V}.
\end{align*}

Applying the bound of $E[\|\nabla f(z_{t})\|^2] \leq \sum_{i=1}^N p_i G_i^2 < \mathcal{V}$ and $E[\|m_{t-1}\|^2] \leq  \mathcal{V}$ for $\forall t\in\{0,...,T-1\}$, we have  

\begin{align*}
    E[\langle\nabla f(z_{t}),  (\frac{1}{  (v_{t-1}+\epsilon)^p}-\frac{1}{(v_{t}+\epsilon)^p}) \odot m_{t-1} \rangle] 
    &\leq E[\|\nabla f(z_{t})\|\| m_{t-1}\|][\sum_{j=1}^d \frac{1}{(v_{t-1,j}+\epsilon)^p}-\frac{1}{(v_{t,j}+\epsilon)^p}]\\
    &\leq \sqrt{\sum_{i=1}^N p_i G_i^2 \cdot  \mathcal{V}} [\sum_{j=1}^d \frac{1}{(v_{t-1,j}+\epsilon)^p}-\frac{1}{(v_{t,j}+\epsilon)^p}]\\
    &\leq   \mathcal{V} [\sum_{j=1}^d \frac{1}{(v_{t-1,j}+\epsilon)^p}-\frac{1}{(v_{t,j}+\epsilon)^p}]
\end{align*}

Then we derive,
\begin{align*}
    E[f(z_{t+1})-f(z_t)] 
    &\leq \frac{\eta \beta_1}{1-\beta_1}
     \mathcal{V} [\sum_{j=1}^d \frac{1}{(v_{t-1,j}+\epsilon)^p}-\frac{1}{(v_{t,j}+\epsilon)^p}]+\frac{1}{2}L^2\eta^2\mu_4^2 (\frac{\beta_1}{1-\beta_1})^2 \mathcal{V} \\ &+\frac{1}{2}\eta^2\mu_4^2\mathcal{V} - E[\langle\nabla f(x_t), \frac{\eta}{ (v_{t}+\epsilon)^p} \odot \Delta_t\rangle]\\
     &+ \frac{L\eta^2\beta_1^2 \mathcal{V}}{(1-\beta_1)^2}E[\sum_{j=1}^d (\frac{1}{(v_{t-1,j}+\epsilon)^p})^2-(\frac{1}{(v_{t,j}+\epsilon)^p})^2] + L\eta^2 \mu_4^2 \mathcal{V}
\end{align*}

$$E[\langle\nabla f(x_t), \frac{\eta}{ (v_{t}+\epsilon)^p} \odot \Delta_t\rangle] = E[\langle\nabla f(x_t), \frac{\eta K \gamma_t}{ (v_{t}+\epsilon)^p} \odot \nabla f(x_t) \rangle - \langle\nabla f(x_t), \frac{\eta}{ (v_{t}+\epsilon)^p} \odot (K \gamma_t \nabla f(x_t) - \Delta_t)\rangle]$$

\begin{align*}
    E[f(z_{t+1})-f(z_t)]  
      &\leq\frac{\eta \beta_1}{1-\beta_1}
     \mathcal{V} [\sum_{j=1}^d \frac{1}{(v_{t-1,j}+\epsilon)^p}-\frac{1}{(v_{t,j}+\epsilon)^p}] +\frac{1}{2}L^2\eta^2\mu_4^2 (\frac{\beta_1}{1-\beta_1})^2 \mathcal{V} +\frac{1}{2}\eta^2\mu_4^2\mathcal{V} \\
     &- \frac{\eta K \gamma_t \mu_3}{2} \|\nabla f(x_t)\|^2+  \frac{\eta\mu_4}{2}  (5K^2 \gamma_t^3 (\sum_{i=1}^N p_i \sigma_i^2 + 2K\sigma_g^2) + 10K^3\gamma_t^3 E[\|\nabla f( x_t)\|^2])\\
     &+ \frac{L\eta^2\beta_1^2 \mathcal{V}}{(1-\beta_1)^2}E[\sum_{j=1}^d (\frac{1}{(v_{t-1,j}+\epsilon)^p})^2-(\frac{1}{(v_{t,j}+\epsilon)^p})^2] + L\eta^2 \mu_4^2\mathcal{V}\\
   &\leq\frac{\eta \beta_1}{1-\beta_1}
     \mathcal{V} [\sum_{j=1}^d \frac{1}{(v_{t-1,j}+\epsilon)^p}-\frac{1}{(v_{t,j}+\epsilon)^p}] +\frac{1}{2}L^2\eta^2\mu_4^2 (\frac{\beta_1}{1-\beta_1})^2 \mathcal{V} +\frac{3}{2}\eta^2\mu_4^2\mathcal{V} \\
     &- (\frac{\eta K \gamma_t \mu_3}{2} - 5\eta \mu_4 K^3 \gamma_t^3 )\|\nabla f(x_t)\|^2+  \frac{5\eta\mu_4 K^2 \gamma_t^3}{2}   (\sum_{i=1}^N p_i \sigma_i^2 + 2K\sigma_g^2)\\
     &+ \frac{L\eta^2\beta_1^2 \mathcal{V}}{(1-\beta_1)^2}E[\sum_{j=1}^d (\frac{1}{(v_{t-1,j}+\epsilon)^p})^2-(\frac{1}{(v_{t,j}+\epsilon)^p})^2] 
\end{align*}
Then, we have, 
\begin{align*}
   & (\frac{\eta K \gamma_t \mu_3}{2} - 5\eta \mu_4 K^3 \gamma_t^3 )\|\nabla f(x_t)\|^2 \\
    & \leq E[f(z_t) -f(z_{t+1})] +\frac{\eta \beta_1}{1-\beta_1}
     \mathcal{V} [\sum_{j=1}^d \frac{1}{(v_{t-1,j}+\epsilon)^p}-\frac{1}{(v_{t,j}+\epsilon)^p}] +\frac{1}{2}L^2\eta^2\mu_4^2 (\frac{\beta_1}{1-\beta_1})^2 \mathcal{V} +\frac{3}{2}\eta^2\mu_4^2\mathcal{V} \\
     &+   \frac{5\eta\mu_4 K^2 \gamma_t^3}{2}   (\sum_{i=1}^N p_i \sigma_i^2 + 2K\sigma_g^2)
     + \frac{L\eta^2\beta_1^2 \mathcal{V}}{(1-\beta_1)^2}E[\sum_{j=1}^d (\frac{1}{(v_{t-1,j}+\epsilon)^p})^2-(\frac{1}{(v_{t,j}+\epsilon)^p})^2] 
\end{align*}

Require $\frac{\eta K \gamma_t \mu_3}{2} - 5\eta \mu_4 K^3 \gamma_t^3 > 0$, then $\gamma_t < \sqrt{\frac{\mu_3}{10\mu_4 K^2}} $. We further derive the bound of inner loop stepsize as $\gamma_t < \min\{ \frac{1}{8LK},\sqrt {\frac{\mu_3}{10\mu_4 K^2}} \}$, or we can get $K \gamma_t < \min\{\frac{1}{8L}, \sqrt{\frac{\mu_3}{10\mu_4}}\}$.

Let $\frac{\frac{5\eta\mu_4 K^2 \gamma_t^3}{2} }{ \frac{\eta K \gamma_t \mu_3}{2} - 5\eta \mu_4 K^3 \gamma_t^3 } = \frac{1}{T}$, then we have $\gamma_t = \sqrt{\frac{\mu_3}{5T\mu_4 K +10 \mu_4 K^2}} = \Theta (\sqrt{\frac{\mu_3}{\mu_4 K T}})$, 
$$ \frac{1}{\frac{\eta K \gamma_t \mu_3}{2} - 5\eta \mu_4 K^3 \gamma_t^3}  = O(\frac{1}{T\eta \mu_4 K^2 \gamma_t^3}) = O(\sqrt{\frac{\mu_4 T}{\mu_3^3 \eta K}})$$

\begin{align*}
      \|\nabla f(x_t)\|^2 
    & \leq O(\sqrt{\frac{\mu_4 T}{\mu_3^3 \eta K}}) E[f(z_t) -f(z_{t+1})] + O(\sqrt{\frac{\mu_4 T}{\mu_3^3 \eta K}})\frac{\eta \beta_1}{1-\beta_1}
     \mathcal{V} [\sum_{j=1}^d \frac{1}{(v_{t-1,j}+\epsilon)^p}-\frac{1}{(v_{t,j}+\epsilon)^p}]\\ &+O(\sqrt{\frac{\mu_4 T}{\mu_3^3 \eta K}}) (\frac{1}{2}L^2\eta^2\mu_4^2 (\frac{\beta_1}{1-\beta_1})^2 \mathcal{V} +\frac{3}{2}\eta^2\mu_4^2\mathcal{V} )\\
     &+ O(\frac{1}{T})(\sum_{i=1}^N p_i \sigma_i^2 + 2K\sigma_g^2)
     +O(\sqrt{\frac{\mu_4 T}{\mu_3^3 \eta K}})\frac{L\eta^2\beta_1^2 \mathcal{V}}{(1-\beta_1)^2}E[\sum_{j=1}^d (\frac{1}{(v_{t-1,j}+\epsilon)^p})^2-(\frac{1}{(v_{t,j}+\epsilon)^p})^2] 
\end{align*}

Sum from $t =0$ to $T-1$ and divide by $T$,  because $z_0 = x_0$ and Lemma \ref{bound1},
\begin{align*}
    \frac{1}{T}\sum_{t=0}^{T-1} \|\nabla f(x_t)\|^2 
     &\leq   O(\sqrt{\frac{\mu_4 }{\mu_3^3 \eta KT}} [f(x_0) -f^*] 
     + \mathcal{V} \sqrt{\frac{\mu_4 \eta}{\mu_3^3  K T}}
      \sum_{j=1}^d (\mu_4 - \mu_3)
     + \sqrt{\frac{\mu_4^5 \eta^3 T}{\mu_3^3 K}}\mathcal{V} \\
     &+ \frac{1}{T}(\sum_{i=1}^N p_i \sigma_i^2 + 2K\sigma_g^2)
     +  \sqrt{\frac{\mu_4 \eta^3}{\mu_3^3  K T}} \mathcal{V}\sum_{j=1}^d (\mu_4^2 -\mu_3^2))
\end{align*}

Notice that  $\mathcal{V} = O(K^2  \gamma_t^2 (1+\frac{1}{S}))$, neglect constant $\sigma_i$, $\sigma_g$,
\begin{align*}
    \frac{1}{T}\sum_{t=0}^{T-1}  \|\nabla f(x_{t})\|^2 & \leq  O( \sqrt{\frac{\mu_4 }{\mu_3^3 \eta K T}} + \frac{K\sigma_g^2}{T} 
     +  (1+\frac{1}{S})(\sqrt{\frac{\mu_4 \eta K}{\mu_3  T^3}}
     + \sqrt{\frac{\mu_4^3 \eta^3 K}{\mu_3 T}}  
     +  \sqrt{\frac{\mu_4^3 K \eta^3  }{\mu_3 T^3}}))\\
     &\leq  O( \sqrt{\frac{\mu_4 }{\mu_3^3 \eta K T}} + \frac{K\sigma_g^2}{T} 
     +  (1+\frac{1}{S})\sqrt{\frac{\mu_4^3 \eta^3 K}{\mu_3 T}}) .
\end{align*}

Our $p$-FedAdam methods are proved to converge with convergence rate of $O(\frac{1}{T})$. And we get the result that, if $K \gamma_t <  \min\{ \frac{1}{8L },\sqrt{\frac{\mu_3}{10\mu_4}} \}$, Federated AGMs always converge.

Consider the calibration parameter $(\mu_{lower}, \mu_{upper})$ in FedAdam, and we further require $\eta = \frac{1}{K}$;

for $p$-FedAdam, $K \gamma_t <  O(\min\{ \frac{1}{L }, (\frac{\epsilon}{  1+\frac{1}{S}})^{\frac{p}{2+2p}} \})$,
$$ \min_{t=0,...,T-1} \|\nabla f(x_{t})\|^2 \leq   O( \sqrt{\frac{(1+1/S)^{3p}}{\epsilon^p T}} +\frac{K\sigma_g^2}{T}+ \sqrt{\frac{(1+1/S)^{2+p}}{\epsilon^{3p} K^2 T}}).$$
\end{proof}

Thus, we get the sublinear convergence rate of $p$-FedAdam methods in nonconvex setting with full device participation,  which can also recover $\epsilon$-FedAdam methods. 

\subsection{$s$-FedAdam Convergence in Nonconvex Setting}

As $s$-FedAdam also has constrained bound pair $(\mu_5, \mu_6)$, we can learn from the proof of $p$-FedAdam method.

\begin{lemma}\label{z_t_2}
Define $z_{t} = x_{t} + \frac{\beta_1}{1-\beta_1}(x_{t} - x_{t-1}), \forall t \geq 1$ $\beta_1 \in [0,1)$. Then the following updating formula holds for $s$-FedAdam optimizer: 
\begin{equation}
    z_{t+1}=z_{t} + \frac{ \eta \beta_1}{1-\beta_1}
    (\frac{1}{softplus(\sqrt{v_{t-1}})}-\frac{1}{softplus(\sqrt{v_t})})\odot m_{t-1}-\frac{\eta}{softplus(\sqrt{v_t})}\odot \Delta_{t} ;
\end{equation}
\end{lemma}

\begin{proof}
\begin{align*}
    z_{t+1} &=x_{t+1} + \frac{\beta_1}{1-\beta_1}(x_{t+1} - x_{t})\\
    z_{t+1} &= z_{t} +\frac{1}{1-\beta_1}(x_{t+1} - x_{t})-\frac{\beta_1}{1-\beta_1}(x_{t} - x_{t-1})\\
    &=z_{t} -\frac{1}{1-\beta_1}\frac{\eta}{softplus(\sqrt{v_t})} \odot m_t +\frac{\beta_1}{1-\beta_1}\frac{\eta}{softplus(\sqrt{v_{t-1}})}\odot  m_{t-1}\\
    &=z_{t} + \frac{\eta \beta_1}{1-\beta_1}
    (\frac{1}{softplus(\sqrt{v_{t-1}})}-\frac{1}{softplus(\sqrt{v_{t}})})\odot m_{t-1}-\frac{\eta}{softplus(\sqrt{v_t})}\odot \Delta_t
\end{align*}
\end{proof}

Then we can derive the same result as $\epsilon$-FedAdam format.

\begin{lemma}\label{square_2}
As defined in Lemma \ref{z_t_2}, with the condition that $v_t\geq v_{t-1}$, we can derive the bound of distance of $\|z_{t,k+1}-z_{t,k}\|^2$ as follows:
\begin{align}\label{eq:distance}
    E[\|z_{t+1}-z_{t}\|^2]&\leq \frac{2\eta^2 \beta_1^2 }{(1-\beta_1)^2}\mathcal{V}  E[\sum_{j=1}^d (\frac{1}{softplus(\sqrt{v_{t-1}})})^2-(\frac{1}{softplus(\sqrt{v_t})})^2]
    +2\eta^2 \mu_6^2\mathcal{V}.
\end{align}
\end{lemma}

\begin{proof}
Since softplus function is monotone incereasing function, we can similarly prove it as the way in Lemma \ref{square}.
\end{proof}

\begin{lemma}\label{product_2}
As defined in Lemma \ref{z_t_2}, with the condition that $v_t\geq v_{t-1}$, we can derive the bound of the inner product as follows:
\begin{align} 
    -&E[\langle \nabla f(z_{t})-\nabla f(x_{t}), \frac{\eta}{softplus(\sqrt{v_t})}\odot \Delta_{t}\rangle]  
    \leq \frac{1}{2}L^2\eta^2 \mu_6^2 (\frac{\beta_1}{1-\beta_1})^2\mathcal{V} 
+\frac{1}{2}\eta^2\mu_6^2\mathcal{V}.
\end{align}
\end{lemma}

\begin{proof}
We can similarly prove it as the way in Lemma \ref{product}.
\end{proof}

\begin{lemma} \label{lemma:drift_2}
\begin{align}
    E[\langle \nabla& f(x_t), \frac{\eta}{ softplus(\sqrt{v_t})} \odot(K \gamma_t\nabla f(x_t) - \Delta_t)\rangle] \nonumber\\
     &\leq \frac{\eta \gamma_t K}{2} \|\frac{\nabla f(x_t)}{\sqrt{softplus(\sqrt{v_t})}}\|^2+  \frac{L \eta\mu_6}{2}  (5K^2 \gamma_t^3 (\sum_{i=1}^N p_i \sigma_i^2 + 2K\sigma_g^2) + 10K^3 \gamma_t^3 E[\|\nabla f( x_t)\|^2]).
\end{align}
\end{lemma}
\begin{proof}
We can similarly prove it as the way in Lemma \ref{lemma:drift}.
\end{proof}

Similar to the $\epsilon$-FedAdam and $p$-FedAdam proof,  we then can get the analysis of $s$-FedAdam method.

\begin{proof} Start from L-smoothness and by use the above prepared lemmas, we similarly have,
Then, we have, 
\begin{align*}
    &(\frac{\eta K\gamma_t \mu_5}{2} - 5\eta \mu_6 K^3 \gamma_t^3 )\|\nabla f(x_t)\|^2 \\
    & \leq E[f(z_t) -f(z_{t+1})] +\frac{\eta \beta_1}{1-\beta_1}
     \mathcal{V} [\sum_{j=1}^d \frac{1}{softplus(v_{t-1,j})}-\frac{1}{softplus(v_{t,j})}] \\
     &+\frac{1}{2}L^2\eta^2\mu_6^2 (\frac{\beta_1}{1-\beta_1})^2 \mathcal{V} +\frac{3}{2}\eta^2\mu_6^2\mathcal{V} 
     +   \frac{5\eta\mu_6 K^2 \gamma_t^3}{2}   (\sum_{i=1}^N p_i \sigma_i^2 + 2K\sigma_g^2)\\
     &+ \frac{L\eta^2\beta_1^2 \mathcal{V}}{(1-\beta_1)^2}E[\sum_{j=1}^d (\frac{1}{softplus(v_{t-1,j})})^2-(\frac{1}{softplus(v_{t,j})})^2] 
\end{align*}

Require  inner loop stepsize  $K \gamma_t < \min\{ \frac{1}{8L},\sqrt{\frac{\mu_5}{10\mu_6}} \}$.

Let $\frac{\frac{5\eta\mu_6 K^2 \gamma_t^3}{2} }{ \frac{\eta K\gamma_t\mu_5}{2} - 5\eta \mu_6 K^3 \gamma_t^3 } = \frac{1}{T}$, then we have $\gamma_t = \sqrt{\frac{\mu_5}{5T\mu_6 K +10 \mu_6 K^2}} = \Theta (\sqrt{\frac{\mu_5}{\mu_6 K T}})$, 
$$ \frac{1}{\frac{\eta K\gamma_t\mu_5}{2} - 5\eta \mu_6 K^3 \gamma_t^2 }  = O(\frac{1}{T\eta \mu_6 K^2 \gamma_t^3}) = O(\sqrt{\frac{\mu_6 T}{ \mu_5^3 \eta K}})$$

\begin{align*}
      \|\nabla f(x_t)\|^2 
    & \leq  O(\sqrt{\frac{\mu_6 T}{ \mu_5^3 \eta K}}) E[f(z_t) -f(z_{t+1})] +  O(\sqrt{\frac{\mu_6 T}{ \mu_5^3 \eta K}}) \frac{\eta \beta_1}{1-\beta_1}
     \mathcal{V} [\sum_{j=1}^d \frac{1}{softplus(v_{t-1,j})}-\frac{1}{softplus(v_{t,j})}]\\ 
     &+ O(\sqrt{\frac{\mu_6 T}{ \mu_5^3 \eta K}}) (\frac{1}{2}L^2\eta^2\mu_6^2 (\frac{\beta_1}{1-\beta_1})^2 \mathcal{V} +\frac{3}{2}\eta^2\mu_6^2\mathcal{V} )\\
     &+ O(\frac{1}{T})(\sum_{i=1}^N p_i \sigma_i^2 + 2K\sigma_g^2)
     +  O(\sqrt{\frac{\mu_6 T}{ \mu_5^3 \eta K}}) \frac{L\eta^2\beta_1^2 \mathcal{V}}{(1-\beta_1)^2}E[\sum_{j=1}^d (\frac{1}{softplus(v_{t-1,j})})^2-(\frac{1}{softplus(v_{t,j})})^2] 
\end{align*}

Sum from $t = 1$ to $T $ and divide by $T$,  because $z_0 = x_0$ and Lemma \ref{bound1},
\begin{align*}
    \frac{1}{T}\sum_{t=0}^{T-1} \|\nabla f(x_t)\|^2  
     &\leq  O([f(x_0) -f^*]  \sqrt{\frac{\mu_6}{ \mu_5^3 \eta K T}}  +\mathcal{V} \sqrt{\frac{\mu_6 \eta}{ \mu_5^3 K T}}
      \sum_{j=1}^d (\mu_6 - \mu_5)
     +\sqrt{\frac{\mu_6^5 \eta^3 T}{ \mu_5^3 K }}\mathcal{V} \\
     &+ \frac{1}{T}(\sum_{i=1}^N p_i \sigma_i^2 + 2K\sigma_g^2)
     + \sqrt{\frac{\mu_6 \eta^3}{ \mu_5^3 K T}}\mathcal{V} \sum_{j=1}^d (\mu_6^2 -\mu_5^2))
\end{align*}

Notice that $\mathcal{V} = O(K^2\gamma_t^2(1+\frac{1}{S}))$, neglect constant $\sigma_i$, $\sigma_g$,
\begin{align*}
    \frac{1}{T}\sum_{t=0}^{T-1}  \|\nabla f(x_{t})\|^2  &\leq  O( \sqrt{\frac{\mu_6 }{\mu_5^3 \eta K T}} + \frac{K\sigma_g^2}{T} 
     +  (1+\frac{1}{S})(\sqrt{\frac{\mu_6 \eta K}{\mu_5  T^3}}
     + \sqrt{\frac{\mu_6^3 \eta^3 K}{\mu_5 T}}  
     +  \sqrt{\frac{\mu_6^3 K \eta^3  }{\mu_5 T^3}}))\\
     &\leq  O( \sqrt{\frac{\mu_6 }{\mu_5^3 \eta K T}} + \frac{K\sigma_g^2}{T} 
     +  (1+\frac{1}{S})\sqrt{\frac{\mu_6^3 \eta^3 K}{\mu_5 T}}) .
\end{align*}

Our $s$-FedAdam methods are proved to converge with convergence rate of $O(\frac{1}{T})$. And we get the result that, if $K \gamma_t <  \min\{ \frac{1}{8L },\sqrt{\frac{\mu_5}{10\mu_6}} \}$, $s$-FedAdam  always converges.

Consider the calibration parameter $(\mu_{lower}, \mu_{upper})$ in FedAdam, and we further require $\eta = \frac{1}{K}$, $ K \gamma_t < O( \min\{ \frac{1}{L },\sqrt[3]{\frac{1 }{\beta  \sqrt{1+\frac{1}{S}} }} \})$,
$$ \min_{t=0,...,T-1} \|\nabla f(x_{t})\|^2 \leq  O( \sqrt{\frac{\beta (1+1/S)^{3/2}}{T}}+ \frac{K\sigma_g^2}{T} + \sqrt{\frac{\beta^3 (1+1/S)^{5/2}}{K^2 T}}).$$
\end{proof}

Thus, we get the sublinear convergence rate of $s$-FedAdam in nonconvex setting with partial device participation, which is similar to the convergence rate of $\epsilon$-FedAdam and $p$-FedAdam, and the convergence rate is highly related with calibration parameters. Consider the special choice of $\epsilon, p, \beta$, $s$-FedAdam seems to enjoy a better convergence rate.

\bibliographystyle{plain}
\bibliography{FedLearning} 

\end{document}